%% file: white_hole.tex
\documentclass[11pt]{article}

\usepackage[utf8]{inputenc}
\usepackage[T1]{fontenc}
\usepackage{lmodern}
\usepackage[margin=1in]{geometry}
\usepackage{amsmath,amssymb,amsthm,mathtools}
\usepackage{graphicx}
\usepackage{booktabs}
\usepackage{multirow}
\usepackage{array}
\usepackage[table]{xcolor}
\usepackage{tikz}
\usetikzlibrary{positioning,arrows.meta,shapes,fit,calc,backgrounds}
\usepackage{hyperref}
\usepackage{etoolbox}
\usepackage{appendix}
\usepackage[capitalise,noabbrev]{cleveref}
\usepackage{enumitem}
\usepackage{caption}
\usepackage{subcaption}
\usepackage{microtype}
\usepackage{listings}
\usepackage{algorithm}
\usepackage{algpseudocode}
\usepackage{xspace}
\usepackage[numbers,square,sort&compress]{natbib}

\hypersetup{
  colorlinks=true,
  linkcolor=blue!50!black,
  citecolor=blue!50!black,
  urlcolor=blue!50!black,
}

\AtBeginEnvironment{appendices}{\crefalias{section}{appendix}}

\newcommand{\method}{\textsc{HyperQuant}\xspace}
\newcommand{\methodlong}{Hadamard, optimallY Packing, Entropy Rice-coding}
\newcommand{\R}{\mathbb{R}}
\newcommand{\Z}{\mathbb{Z}}
\newcommand{\E}{\mathbb{E}}
\newcommand{\Var}{\operatorname{Var}}
\newcommand{\Cov}{\operatorname{Cov}}
\newcommand{\Vor}{\mathcal{V}}
\newcommand{\Latt}{\Lambda}
\newcommand{\Qfn}{Q_{\Latt}}
\newcommand{\modproj}{\pi_{\Latt}}
\newcommand{\inner}[2]{\langle #1,\, #2 \rangle}
\newcommand{\norm}[1]{\lVert #1 \rVert}

\newcommand{\snr}{\mathrm{SNR}}
\newcommand{\dB}{\mathrm{dB}}
\newcommand{\bff}{\textsc{bf16}\xspace}
\newcommand{\fpEightLong}{\textsc{fp8-e4m3}\xspace}
\newcommand{\Eeightmzero}{\textsc{e8m0}\xspace}
\newcommand{\fpEight}{\textsc{fp8}\xspace}
\newcommand{\fpFour}{\textsc{fp4}\xspace}
\newcommand{\intEight}{\textsc{int8}\xspace}
\newcommand{\intFour}{\textsc{int4}\xspace}
\newcommand{\nvfp}{\textsc{nvfp4}\xspace}
\newcommand{\mxfp}{\textsc{mxfp4}\xspace}
\newcommand{\Eone}{E_{8}} 
\newcommand{\Dfour}{D_{4}} 
\newcommand{\Atwo}{A_{2}} 
\newcommand{\Eoneint}{E_{8}^{\mathrm{int}}}
\newcommand{\Dfourint}{D_{4}^{\mathrm{int}}}
\newcommand{\Atwoint}{A_{2}^{\mathrm{int}}}
\newcommand{\Zoneint}{\Z_{1}^{\mathrm{int}}}
\newcommand{\ppl}{\mathrm{PPL}}
\newcommand{\Unif}{\operatorname{Uniform}}
\newcommand{\diag}{\operatorname{diag}}
\newcommand{\vol}{\operatorname{vol}}

\theoremstyle{plain}
\newtheorem{theorem}{Theorem}
\newtheorem{lemma}{Lemma}
\newtheorem{proposition}{Proposition}
\newtheorem{corollary}{Corollary}

\theoremstyle{definition}
\newtheorem{definition}{Definition}

\theoremstyle{remark}
\newtheorem*{remark}{Remark}

\title{\method: A Rate--Distortion-Optimal Quantization Pipeline for Large
Language and Diffusion Models}

\author{ \begin{tabular}{ccc} Yuval Domb & Hadar Sackstein & Tomer Solberg \\ \multicolumn{3}{c}{\texttt{research@moonmath.ai}} \end{tabular} }
\date{}

\begin{document}
\maketitle

\begin{abstract}
\input{sections/abstract.tex}
\end{abstract}

\setcounter{tocdepth}{2}
{\small\tableofcontents}
\vspace{4pt}\hrule\vspace{8pt}

\input{sections/intro.tex}
\input{sections/prelim.tex}
\input{sections/method.tex}
\input{sections/results.tex}
\input{sections/ablation.tex}
\input{sections/conclusion.tex}

\bibliographystyle{plainnat}
\bibliography{refs}

\begin{appendices}
\appendix
\input{appendix/dither_proof.tex}
\input{appendix/lattice_constants.tex}
\end{appendices}

\end{document}

%% file: sections/abstract.tex
We present \method (\methodlong), a unified post-training quantization pipeline for
the weights and the KV cache of large language and diffusion
transformers. Across a suite of self-contained experiments
(\Cref{tab:abstract-best}), \method outperforms the recent HIGGS
scheme at every operating point from 3 to 5 bits per scalar (bps) on
weights, and beats both TurboQuant and OCTOPUS on KV quantization down
to 1.7 bps. Beyond the LLM setting, \method quantizes the
19B-parameter LTX-2 DiT video model with no observable per-frame
artifacts. End-to-end on an H100 at 4~bps, \method
compresses the linear weights $\sim\!3.9\times$ and the KV cache $\sim\!3.79\times$ at
near-lossless quality.

\method combines four known ideas into a single construction:
(i)~a per-tile Randomized Hadamard Transform that makes the
per-coordinate distribution of weights and activations approximately
Gaussian; (ii)~quantization
to a low-dimensional optimal lattice ($\Eone$, $\Dfour$, $\Atwo$, or
$\Z$); (iii)~lossless bit-stripping and near-entropy-optimal
variable-length Rice coding of the lattice indices; and
(iv)~bias-correction methods for the KV cache that
keep the reconstruction unbiased under inner products, preserving
attention semantics. We further integrate the pipeline with 8-bit and
4-bit Tensor-Core MMA paths (\fpEightLong, \intEight, \nvfp, \mxfp),
and find that \intEight beats \fpEight on the post-RHT lattice output. Project page: \url{https://moonmath.ai/hyperquant/}

\begin{center}\small
\captionof{table}{Typical \method operating points across settings.
Weights/KV/\intEight-MMA rows are Llama-3.1-8B-Instruct on WikiText-2;
the OCTOPUS row is KV-only on Qwen2.5-7B-Instruct (perplexity ($\ppl$)
$\Delta\%$ at $32$-token residual window); LTX-2 is the 19B DiT video
model.}\label{tab:abstract-best}
\begin{tabular}{l c c c c}
\toprule
Setting & rate & criterion & \method & reference \\
\midrule
Weights$+$KV cache, \intEight MMA  & 4 bps & $\ppl$ $\downarrow$ & $7.50$ ($+0.47\%$)  & $7.16$ (\bff) \\
Weights    & 4 bps & $\Delta \ppl\%$ $\downarrow$ & $\mathbf{+3.8\%}$            & $+6.4\%$ (HIGGS) \\
Weights    & 3 bps & $\Delta \ppl\%$ $\downarrow$ & $\mathbf{+22.1\%}$            & $+33\%$ (HIGGS) \\
KV cache  & 2 bps & $\Delta \ppl\%$ $\downarrow$ & $\mathbf{+7.4\%}$ & $+34.7\%$ (OCTOPUS) \\
KV cache  & 2 bps & compression $\uparrow$ & $\mathbf{6.4\times}$ & $3.0\times$ (TurboQuant) \\
KV cache  & 1.7 bps & $\Delta \ppl\%$ $\downarrow$ & $\mathbf{+26.9\%}$ & -- \\
LTX-2 video & 4 bps & LPIPS $\downarrow$ & $0.20$--$0.21$ & $0$ (\bff) \\
\bottomrule
\end{tabular}
\end{center}

%% file: sections/intro.tex
\section{Introduction}\label{sec:intro}

Frontier language and generative models~\citep{llama31,ltx2} routinely exceed tens of billions of parameters and produce KV caches that dominate inference memory at modern context lengths.
Autoregressive decoding is memory-bandwidth-bound: each token requires streaming the entire weight set and KV cache while performing only a thin matrix-vector product, keeping arithmetic intensity well below the hardware's compute-to-bandwidth ratio~\citep{pope2023scaling}.
Post-training quantization (PTQ) turns this bottleneck into a
\emph{rate-distortion} compression problem.

For weight quantization, a long line of work (GPTQ~\citep{frantar2023gptq}, AWQ~\citep{lin2024awq}, SmoothQuant~\citep{xiao2023smoothquant}, OmniQuant~\citep{shao2024omniquant}) has chipped away at the rate, typically at the cost of a calibration dataset and per-layer optimization.
Data-free schemes are simpler and generally preferred for deployment.
The state of the art among them, HIGGS~\citep{malinovskii2025higgs}, identifies two levers: (1)~apply an RHT to weights so their per-coordinate distribution is approximately Gaussian, and (2)~quantize to multi-dimensional codebooks that are MSE-optimal for that Gaussian.
Its headline result, a linearity theorem reducing global perplexity damage to per-layer $\ell_2$ error, justifies focusing the design on per-layer MSE.

HIGGS, however, leaves rate on the table: its codebooks are finite
Lloyd grids with \emph{fixed-rate} indices, and information theory
predicts that an \emph{entropy-coded} quantizer of equal MSE always
needs fewer bits~\citep{lookabaugh1989high}.
Our measurements on real LLM weights (\Cref{sec:results-weights}) confirm the gap: HIGGS's index entropy is $0.6$--$5.9\,\%$ below its fixed bit budget at 3--5~bps.
Lattice coding theory provides the solution~\citep{conway1999sphere,zamir2014lattice}: combine a \emph{lattice} quantizer with a variable-length code over its indices.

KV-cache quantization has converged on a different rotation-plus-marginal scheme.
TurboQuant~\citep{zandieh2025turboquant} rotates each head's KV vector, exploits the Beta marginal of the resulting unit-norm coordinates, and Lloyd-quantizes each scalar.
OCTOPUS~\citep{octopus2025} extends the marginal trick to coordinate \emph{triplets} via an octahedral parameterization, gaining nearly an extra bit at 2~bps.
Both are data-free, yet both pay the same fixed-rate overhead as HIGGS, on a different marginal.
The closest lattice-based contemporary is NestQuant~\citep{savkin2025nestquant}, which combines nested Gosset lattices with a calibration-style QA-LDLQ correction; quantizing weights and KV cache to 4~bps, it raises Llama-3-8B perplexity by $\sim 3.9\%$ over its \bff baseline.

\paragraph{Contributions.} We propose \method, a data-free, post-training pipeline that applies the rate-distortion-optimal triplet of per-tile RHT, lattice quantization, and entropy coding to both the weights and the KV cache. It integrates with Hopper's \fpEight/\intEight and Blackwell's \nvfp/\mxfp MMA paths~\citep{micikevicius2022fp8,nvidia2022hopper,nvidia2024blackwell}, and is benchmarked end-to-end on Llama-3.1-8B-Instruct and LTX-2-19B (\Cref{sec:method,sec:implementation}).
Our contributions are:
\begin{itemize}[leftmargin=*,itemsep=2pt]
  \item \textbf{Per-tile Randomized Hadamard Transform (RHT).} Each
    linear layer's input and weight are independently randomly rotated
    in tiles sized to the lattice dimension and the hardware's MMA
    tile, implemented via the RHT (\Cref{sec:prelim-rht}). The RHT
    folds into the preceding LayerNorm/RMSNorm where possible (no
    runtime cost), and is otherwise installed as a forward hook.
  \item \textbf{Lattice quantization, bit-stripping, and Rice coding.}
    We quantize each rotated tile with one of $\Eone$ (8-D), $\Dfour$
    (4-D), $\Atwo$ (2-D), or scalar $\Z$, strip the bits that lattice
    membership fixes deterministically, and encode the resulting
    indices with a Rice code calibrated on the per-norm Gaussian; the
    realized rate then lands within $\sim\!0.01$~bps of any requested
    target (\Cref{sec:results-weights}).
  \item \textbf{A bias-correction menu for the KV cache:} a per-layer random rotation ($\pm1$ signs or full Quantized
Johnson-Lindenstrauss, QJL~\citep{zandieh2024qjl}) and optional
Schuchman subtractive dither~\citep{schuchman1964dither,zamir1992universal}.
    We prove (\Cref{app:dither_proof}) that subtractive lattice dither
    is \emph{strictly} inner-product unbiased on every cached vector,
    unlike the distribution-average unbiasedness of QJL's 1-bit
    sketch~\citep{zandieh2024qjl}.
  \item \textbf{A rate-distortion decomposition of the $\Atwo$-vs-HIGGS
    gap at matched bps} (\Cref{sec:results-weights}). At 4~bps, a
    $\sim\!0.75\,\dB$ piece is HIGGS's fixed-rate index redundancy
    (recoverable by any entropy-coded retrofit), and a $\sim\!0.36\,\dB$
    piece is the structural advantage of an unbounded codebook over any
    finite one, enabled by variable-length coding; by 5~bps the latter
    grows to $\sim\!1.79\,\dB$ while the index-entropy piece shrinks to
    $\sim\!0.18\,\dB$. Switching from $\Atwo$ to \method's default
    $\Eone$ adds a further $\sim\!0.49\,\dB$ granular gain, a purely
    geometric advantage that is most pronounced at low rates ($0.53$~PPL
    at 3~bps) and negligible above 4.25~bps.
  \item \textbf{A two-regime characterization of KV-cache quantization
    quality} (\Cref{sec:results-kv}): a high-quality regime ($\ge 2.5$~bps)
    where all bias-correction variants lie within $0.04$~PPL, and a
    high-compression regime ($1.7$--$2.5$~bps) where QJL-style rotation
    pulls ahead by up to $\sim\!0.5$~PPL.
  \item \textbf{An end-to-end stress test on the 19B-parameter LTX-2
    video DiT,} showing that the same pipeline transfers to a non-LLM
    transformer architecture and delivers $3.7\times$ weight
    compression with no perceptible quality loss
    (\Cref{sec:results-ltx}).
\end{itemize}

\paragraph{Outline.} \Cref{sec:prelim} reviews the classical ingredients (RHT, optimal low-dimensional lattices, Rice coding, and dithering); \Cref{sec:method} assembles them into the \method pipeline, adding the bit-stripping transform that makes Rice coding near rate-optimal, and \Cref{sec:implementation} covers its implementation. \Cref{sec:results,sec:ablation} give benchmark comparisons against HIGGS, TurboQuant, and OCTOPUS, together with per-component ablations. \Cref{sec:conclusion} concludes and suggests future directions.

\section{Related work}\label{sec:related}

\paragraph{Weight quantization with rotations and finite codebooks.}
HIGGS~\citep{malinovskii2025higgs} is the data-free state of the art:
RHT plus a multi-dimensional Lloyd codebook with
fixed-rate indices. \method shares this architecture but replaces the
finite Lloyd codebook with an infinite lattice (codeword density set by
a continuous SNR knob) and bit-strips and entropy-codes the index
stream with a Rice code rather than transmitting at a fixed $\log_2 N$
bits per index. \Cref{sec:results-weights} quantifies both differences.

\paragraph{Calibration-based methods.} GPTQ~\citep{frantar2023gptq},
AWQ~\citep{lin2024awq}, OmniQuant~\citep{shao2024omniquant},
SmoothQuant~\citep{xiao2023smoothquant}, SpQR~\citep{dettmers2023spqr},
and SliceGPT~\citep{ashkboos2024slicegpt} use calibration data to refine
per-channel scales, salvage outlier features, or solve a Hessian-aware
weight-allocation problem. \method is data-free by design and orthogonal
to these methods; composing its bit-allocation knob with an LDLQ-style
calibration update~\citep{tseng2024quipsharp,savkin2025nestquant} is a
natural future direction.

\paragraph{KV-cache quantization.} TurboQuant~\citep{zandieh2025turboquant}
rotates each head's KV vector and Lloyd-Max scalar-quantizes the
resulting Beta-distributed coordinates, reaching $4$--$7\times$
compression with near-zero quality loss. OCTOPUS~\citep{octopus2025}
extends this to triplets via an octahedral parameterization, pushing
to extreme ($\le 2$-bit) operating points. Both stay strictly scalar
(or 3-D) after rotation; \method instead uses true multi-dimensional
lattices ($\Eone$ is 8-D) with a variable-length code, giving higher
granular gain and an unbounded codebook. On the bias side, both
TurboQuant and OCTOPUS offer only distribution-average unbiasedness;
subtractive dither~\citep{schuchman1964dither,zamir1992universal,erez2005closeness},
which \method adopts, is strictly per-vector unbiased, as we prove in
\Cref{app:dither_proof}.

\paragraph{Nested-lattices.} NestQuant~\citep{savkin2025nestquant} is
the closest contemporary: like \method it uses the $\Eone$ lattice,
but relies on calibration-style QA-LDLQ post-processing. \method stays data-free, substituting an entropy code and a richer
rotation menu; QA-LDLQ is orthogonal to our design and could be
composed with \method as a future calibration step. On a
baseline-normalized basis, \method's data-free W$+$KV path costs
$+4.6\%$ at $4$~bps on Llama-3.1-8B, within a fraction of a point of
NestQuant's $+3.9\%$ obtained \emph{with} QA-LDLQ; NestQuant's own
ablation shows that removing QA-LDLQ raises its cost to $+7.6\%$.

\paragraph{Diffusion and video transformers.} Quantization of
diffusion transformers~\citep{peebles2023dit,esser2024sd3,ltx2} is
less explored than LLM quantization, and most published numbers
target image rather than video models. The
LTX-2~\citep{ltx2} stress test in \Cref{sec:results-ltx} is, to our
knowledge, the first end-to-end PTQ result on a billion-parameter
video DiT, complementing earlier OCTOPUS results~\citep{octopus2025}
on the Wan-1.3B DiT.

\paragraph{Numerical formats.} \fpEight was standardized on NVIDIA's
H100~\citep{micikevicius2022fp8,nvidia2022hopper}; the smaller \fpFour
formats (\nvfp and OCP \mxfp) target the Blackwell
generation~\citep{rouhani2023mxfp,ocpmxfp,nvidia2024blackwell}. The
two \fpFour formats differ in scale encoding (\fpEightLong in \nvfp vs.\
power-of-2 \Eeightmzero in \mxfp) and block size ($16$ vs.\ $32$). \method
targets both; our experiments show \nvfp is the only one
quality-viable for KV quantization (\Cref{sec:results-fp4}).

%% file: sections/prelim.tex
\section{Preliminaries}\label{sec:prelim}

This section reviews the four classical ingredients the rest of the
paper builds on: the RHT,
optimal low-dimensional lattices as vector quantizers, Rice
entropy coding, and the subtractive-dither and random-rotation
bias-correction mechanisms. The material is well-established
and included to fix notation. \Cref{sec:method} assembles these
ingredients into \method, identifying the design choices that diverge
from the classical constructions.

\subsection{Randomized Hadamard Transform}\label{sec:prelim-rht}

The RHT composes the Walsh-Hadamard matrix $H_n$ (an $n\times n$
orthonormal matrix, $O(n\log n)$ Cooley-Tukey butterfly) with a
random sign diagonal $D = \operatorname{diag}(\pm 1)$:
\begin{equation}\label{eq:rht}
  \text{RHT}_n(x) \;=\; H_n D \, x, \qquad D_{ii} \in \{-1, +1\}
  \text{ iid uniform.}
\end{equation}
Two properties are key. First, RHT is a fast
Johnson-Lindenstrauss-style mixer: applying $H_n D$ to any
deterministic $x \in \R^n$ yields a vector whose
\emph{empirical} distribution is, with high probability, close to
$\mathcal{N}(0, \norm{x}^2/n \cdot I)$~\citep{ailon2009fast,dasgupta2010sparse,halko2011finding}.
Second, because $\text{RHT}_n$ is orthogonal, applying it before quantization
and inverting it after preserves $\ell_2$ error, since the lattice cell
volume is the same in the pre- and post-RHT spaces. RHT therefore
\emph{redistributes} the per-coordinate
quantization error from a few outlier coordinates
(in the raw activation space) into an approximately isotropic
spread~\citep{chee2023quip,malinovskii2025higgs}.

\subsection{Optimal low-dimensional lattices}\label{sec:prelim-lattices}

A lattice $\Latt \subset \R^n$ is the set of integer linear
combinations of $n$ basis vectors. Two lattice invariants
characterize its quality as a vector quantizer:
\begin{itemize}[leftmargin=*,itemsep=4pt]
  \item the \emph{normalized second moment}
        \[
          G(\Latt) := \frac{1}{n}\,
          \E_{U \sim \Unif(\Vor(\Latt))}\!\big[\norm{U}^2\big]
          \cdot \det(\Latt)^{-2/n},
        \]
        the per-coordinate mean-squared error of quantizing a
        uniform point of the Voronoi cell $\Vor(\Latt)$
        to the origin, made scale-invariant by the
        $\det(\Latt)^{-2/n}$ factor. It is bounded
        below by $1/(2\pi e)$, attained asymptotically
        by the $n$-dimensional ball, and a smaller $G(\Latt)$
        gives lower granular distortion at fixed
        rate~\citep{conway1999sphere,zamir2014lattice}: intuitively,
        $G(\Latt)$ measures how ``round'' the Voronoi cell is.
  \item the \emph{packing density} $\Delta(\Latt)$, the fraction of
        $\R^n$ covered by non-overlapping balls of radius
        $r_{\text{pack}}(\Latt)$ (the Voronoi inradius,
        half the lattice minimum distance) centered at
        every lattice point. A higher $\Delta(\Latt)$ fits more
        codewords at a fixed minimum
        separation, controlling how densely the
        codebook tiles space at fixed cell radius~\citep{conway1999sphere}.
\end{itemize}
In every dimension $\le 8$ for which the optimum is known, the
densest sphere packing also achieves the smallest known $G(\Latt)$:
$\Z = \mathbb{Z}$ in 1-D, the hexagonal $\Atwo$ in 2-D, the
Schl\"afli lattice $\Dfour$ in 4-D, and the Gosset lattice $\Eone$
in 8-D~\citep{conway1999sphere}.
\Cref{tab:lattice-constants} collects their normalized
second moments and the resulting high-rate gap to the Shannon bound.

\begin{table}[ht]
\centering
\small
\begin{tabular}{l c c c c}
\toprule
Lattice & $n$ & $G(\Latt)$ & High-rate gap to Shannon & Decoder
$O(\cdot)$ \\
\midrule
$\Z$    & 1 & $1/12 = 0.0833$ & $1.53\,\dB$ & $O(1)$ \\
$\Atwo$ & 2 & $0.0802$        & $1.36\,\dB$ & $O(1)$ \\
$\Dfour$& 4 & $0.0766$        & $1.17\,\dB$ & $O(1)$ \\
$\Eone$ & 8 & $0.0717$        & $0.88\,\dB$ & $O(1)$ \\
$\infty$-D sphere & --- & $1/(2\pi e) = 0.0586$ & $0\,\dB$ & --- \\
\bottomrule
\end{tabular}
\caption{Classical lattice quantizer constants for the four
lattices used in this paper. The high-rate gap to the
Shannon bound is $10\log_{10}(G(\Latt) \cdot 2\pi e)$. The
asymptotic limit $1/(2\pi e)$ is the infinite-dimensional sphere
bound. These normalized second moments are scale-invariant, so they
apply unchanged to the integer realizations
$\Eoneint$/$\Dfourint$/$\Atwoint$/$\Zoneint$ used in our
code.}\label{tab:lattice-constants}
\end{table}

\paragraph{Nearest-neighbor decoding.} For $\Atwo$, $\Dfour$, and
$\Eone$, the closest-point algorithm follows
Conway-Sloane~\citep[Ch.~20]{conway1999sphere}: round each
coordinate to the nearest integer; if the result violates the
lattice's parity constraint, move to the nearest point of the
complementary coset (the half-integer coset
$D_8 {+} \tfrac{1}{2}\mathbf{1}$ for $\Eone$, the
parity-flipped neighbor for $\Dfour$ and $\Atwo$), and keep the
candidate with smaller residual norm. The work is $O(1)$ per scalar.
$\Eone$ is the highest-dimensional lattice admitting such a
constant-time closed-form decoder~\citep{conway1999sphere}.

\paragraph{Granular gain.} The advantage of
multi-dimensional vector quantization (VQ) over scalar quantization
is the freedom to choose the shape of the quantization
cell. A scalar quantizer's cell is forced to be an interval, the
Voronoi cell of $\Z$, fixing its normalized second moment
at $G(\Z) = 1/12 \approx 0.0833$. In higher dimensions the optimal
Voronoi cell grows rounder and $G$
descends toward the ball's limit
$G_\infty = 1/(2\pi e) \approx 0.0586$. The ratio
$10\log_{10}(G(\Z)/G(\Latt))$ is the \emph{granular gain} of
$\Latt$ over the scalar quantizer, the source-coding
counterpart of the channel-coding shaping gain from
constellation design~\citep{forney1989multidim,zamir2014lattice}.

The four lattices traverse the $1.53\,\dB$ budget from $\Z$ to
the Shannon bound: $\Atwo$ recovers $0.17\,\dB$,
$\Dfour$ recovers $0.37\,\dB$, and $\Eone$ recovers $0.65\,\dB$
($42\%$ of the total) while retaining an $O(1)$ constant-time
decoder. The 24-D Leech lattice, the best known structured lattice
in its dimension~\citep{conway1999sphere}, adds a further
$\sim\!0.38\,\dB$ ($25\%$) at the cost of a significantly more
complex decoder, still leaving $0.50\,\dB$ to Shannon.
Closing the residual gap requires high-dimensional random lattices,
which asymptotically approach the bound~\citep{zamir2014lattice}
but admit no practical nearest-neighbor decoder.
$\Eone$ is therefore not where the gap closes but where the
gain-per-decoder-complexity curve sharply drops.

\subsection{Entropy coding and Rice codes}\label{sec:prelim-rice}

\paragraph{Entropy coding.} For a discrete source $X$ with
probability mass function $p$, Shannon's source coding theorem
bounds the average code length per symbol of any uniquely
decodable code below by the entropy
\[
  H(X) \;=\; -\sum_{x} p(x)\,\log_2 p(x),
\]
and this bound is achievable to within a fraction of a
bit per symbol by practical coders such as Huffman or
arithmetic coding~\citep{cover2006elements}. Entropy is thus
the rate floor for \emph{lossless} compression of a
discrete source. A \emph{lossy} pipeline like ours
splits into two stages: a quantizer maps a
continuous input to a discrete index, trading distortion for the
index bit count (its rate-distortion behavior),
and a lossless entropy coder then represents the
index stream at an expected rate near its entropy.
The entropy coder neither distorts the source nor changes the
quantizer's distortion; it only realizes the
information-theoretic floor on the indices in actual bits.

\paragraph{Variable-length codes and unbounded alphabets.} A
fixed-length code over an alphabet of size $N$ pays exactly
$\log_2 N$ bits per symbol and is undefined when $N = \infty$. A
variable-length code has no such limit: it
addresses an arbitrary discrete alphabet at finite expected rate
whenever the source entropy is finite. This is the operative
advantage in our setting. A lattice quantizer's output is an integer vector
with unbounded support, so a fixed-length code is not even
well-defined; yet for Gaussian-like inputs the integer histogram has
finite entropy, which a variable-length code attains
at finite cost~\citep{cover2006elements,zamir2014lattice}. We
quantify this advantage empirically in \Cref{sec:results-weights}.

\paragraph{Rice codes.} Among variable-length
codes, the Rice code~\citep{rice1979rice} is the practical
near-optimal choice for sources whose integer histogram is
two-sided geometric (Laplacian on the integer lattice);
it is the power-of-two-parameter specialization of the Golomb
code, optimal in this regime. Given a parameter $k$,
a non-negative integer $m$ is encoded as $\lfloor m/2^k \rfloor$
in unary followed by $m \bmod 2^k$ in $k$ raw bits; signed values
use zig-zag interleaving or an explicit sign bit. The
optimal $k$ for a geometric distribution with parameter $p$ is
\[
  k^{*} \;=\; \bigr\lfloor \log_2 \lceil-\ln (2-p)\,/\,\ln(1-p)\rceil\bigr\rfloor.
\]
The histograms we encounter (lattice indices of
RHT-transformed weights and activations) are not
exactly Laplacian, but close enough that a Rice code with
empirically calibrated $k$ stays within $\sim\!0.1$~bps of the
symbols' marginal entropy across our calibration sweep
(\Cref{app:bps-calibration}). We adopt it throughout: it has
constant per-codeword cost and is
a genuine variable-length code over $\Z$, and the $\sim\!0.1$~bps it
concedes to an ideal marginal coder buys a stateless, table-free
$O(1)$ decoder. No \emph{context} coder can do better,
since the stripped symbols carry essentially no inter-symbol
redundancy (\Cref{app:strip-optimality}).

\subsection{Bias correction: rotation and subtractive dither}
\label{sec:prelim-dither}

A nearest-neighbor lattice quantizer is \emph{deterministic} in its
input, hence biased: the reconstruction $\hat{x} =
\Qfn(x)$ satisfies $\hat{x} = x + e(x)$ with a non-zero, $x$-dependent
error $e(x) \in -\Vor(\Latt)$. For weight quantization
this is harmless: biased reconstructions are
absorbed by surrounding affine parameters and disappear into the
linearity theorem~\citep{malinovskii2025higgs}. For the KV cache,
however, attention
\begin{equation*}
  \mathrm{Attention}(q,k,v) = \mathrm{softmax}\bigl(\tfrac{1}{\sqrt{d}}q^{\top}k\bigr)v
\end{equation*}
is \emph{linear in $k$ and $v$ inside the dot product}, so a
deterministic bias in $k$ accumulates through the softmax denominator
and shifts attention scores systematically.

Two classical mechanisms can remove this bias.

\paragraph{Random rotation (QJL-style).} Apply a Haar-uniform
orthogonal matrix $S \sim \Unif(\mathrm{O}(n))$
before quantization and $S^{\top}$ after:
\begin{equation}\label{eq:qjl-recipe}
  \hat{x}_{\mathrm{rot}}(x; S) \;=\; S^{\top}\, \Qfn(S\,x).
\end{equation}
Averaged over $S$, the error $e_{\mathrm{rot}}(x) = -S^{\top}
\modproj(Sx)$ is zero-mean and isotropic~\citep{zandieh2024qjl,zandieh2025turboquant}.
In deployment, however, $S$
is drawn once per layer and frozen, so the error is deterministic
given $(x, S_0)$ and biased on every individual cached vector. We
formalize this in \Cref{prop:qjl}.

\paragraph{Subtractive dither (Schuchman-Zamir-Feder).} Draw $U
\sim \Unif(\Vor(\Latt))$ fresh on every forward call, independently
of everything else, and reconstruct
\begin{equation}\label{eq:dither-recipe}
  \hat{x}_{\mathrm{dith}}(x; U) \;=\; \Qfn(x + U) - U.
\end{equation}
The error $e_{\mathrm{dith}}(x; U) = -\modproj(x + U)$ is then
\emph{exactly uniform} on $-\Vor(\Latt)$ and independent of $x$, by the
Crypto Lemma~\citep{schuchman1964dither,zamir1992universal}
(self-contained proof in \Cref{app:dither_proof}). In particular,
\begin{equation}\label{eq:dither-mean-zero}
  \E_{U}\bigl[\inner{q}{\hat{x}_{\mathrm{dith}}(x; U)} \,\big|\, x\bigr]
  \;=\; \inner{q}{x}
  \qquad \forall\, q, x \in \R^n.
\end{equation}
We call this \emph{strict, per-vector} inner-product
unbiasedness, in contrast to QJL's averaged-over-$S$ unbiasedness.

\paragraph{Composing rotation and dither.} The two mechanisms are
statistically orthogonal: the rotation is a deterministic-given-$x$
linear map and the dither is independent of $x$, so the
inner-product unbiasedness of \eqref{eq:dither-mean-zero} survives
composition with any orthogonal pre-rotation and any further
deterministic-given-$x$ linear post-processing (\Cref{prop:composition}).
The composed scheme keeps the rotation's isotropic error and the
dither's strict unbiasedness. A standard alternative to the Haar-uniform
rotation is the sign-rotation $S = \operatorname{diag}(\pm
1)$, which costs only $n$ bits per layer and suffices when the source
is approximately exchangeable under coordinate permutations.

%% file: sections/method.tex
\section{The \method design}\label{sec:method}\label{sec:method-steps}

\Cref{fig:pipeline} is the end-to-end \method block diagram and the
map for this section. The \emph{encode} path (top, left to right)
turns a \bff tile into a compact code; the \emph{decode} path
(bottom, right to left) inverts every active block in reverse order
and feeds the low-precision matrix-multiply-accumulate (MMA). A
single encode path serves both linear weights (offline) and
the KV cache (online), differing only in the
two bias-correction blocks, Rotate and Add dither, which run
for the KV cache alone. We cover the blocks in figure
order, each forward block with its inverse under one heading
(marked \emph{KV only} where applicable), folding the
integer-lattice detail into the Quantize and Strip blocks where it is
used.

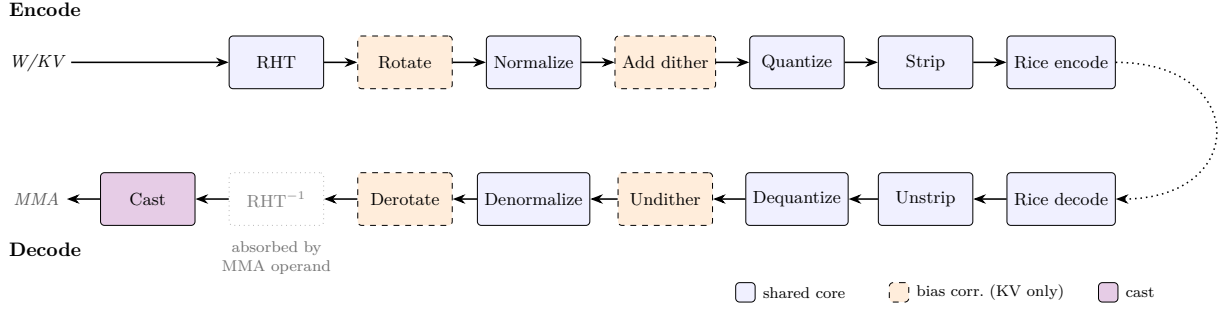
\begin{figure}[t]
\centering
\resizebox{\textwidth}{!}{%
\begin{tikzpicture}[
    >=Stealth,
    node distance=6mm and 6mm,
    core/.style={draw, rounded corners=2pt, minimum width=1.7cm,
                 minimum height=0.95cm, align=center, fill=blue!6,
                 font=\footnotesize},
    bias/.style={draw, rounded corners=2pt, minimum width=1.7cm,
                 minimum height=0.95cm, align=center, dashed,
                 fill=orange!14, font=\footnotesize},
    cast/.style={draw, rounded corners=2pt, minimum width=1.7cm,
                 minimum height=0.95cm, align=center, fill=violet!20,
                 font=\footnotesize},
    ghost/.style={draw, dotted, rounded corners=2pt, minimum width=1.7cm,
                  minimum height=0.95cm, align=center, draw=black!35,
                  text=black!50, font=\footnotesize},
    io/.style={align=center, font=\footnotesize\itshape},
    term/.style={align=center, font=\footnotesize\itshape, text=black!65},
    flow/.style={->, thick},
    tag/.style={font=\scriptsize, text=black!55},
]
\node[core]                  (erht)   {RHT};
\node[bias, right=of erht]   (erot)   {Rotate};
\node[core, right=of erot]   (enorm)  {Normalize};
\node[bias, right=of enorm]  (edith)  {Add dither};
\node[core, right=of edith]  (elat)   {Quantize};
\node[core, right=of elat]   (estrip) {Strip};
\node[core, right=of estrip] (erice)  {Rice encode};
\foreach \a/\b in {erht/erot,erot/enorm,enorm/edith,%
                   edith/elat,elat/estrip,estrip/erice}
  {\draw[flow] (\a) -- (\b);}

\node[core,  below=15mm of erice]  (drice)   {Rice decode};
\node[core]  at (estrip |- drice)  (dstrip)  {Unstrip};
\node[core]  at (elat   |- drice)  (ddeq)    {Dequantize};
\node[bias]  at (edith  |- drice)  (dundith) {Undither};
\node[core]  at (enorm  |- drice)  (dunnorm) {Denormalize};
\node[bias]  at (erot   |- drice)  (dunrot)  {Derotate};
\node[ghost] at (erht   |- drice)  (drht)    {RHT$^{-1}$};
\node[cast,  left=of drht]         (dcast)   {Cast};
\node[term,  left=of dcast]        (dmma)    {MMA};
\foreach \a/\b in {drice/dstrip,dstrip/ddeq,ddeq/dundith,%
                   dundith/dunnorm,dunnorm/dunrot,dunrot/drht,drht/dcast,%
                   dcast/dmma}
  {\draw[flow] (\a) -- (\b);}
\node[tag, below=1.5mm of drht, text width=2.1cm, align=center]
  {absorbed by MMA operand};

\node[io] at (dmma |- erht) (ein) {W/KV};
\draw[flow] (ein) -- (erht);
\node[font=\small\bfseries, anchor=south west]
  at ([yshift=4mm]ein.north west) {Encode};
\node[font=\small\bfseries, anchor=north west]
  at ($(ein.west |- dmma.south) + (0,-4mm)$) {Decode};

\draw[->, thick, dotted] (erice.east)
  to[out=0, in=0, looseness=2.5]
  (drice.east)
;

\node[core, minimum width=3.5mm, minimum height=3.5mm,
      label={[font=\scriptsize]right:shared core}]
  (lg1) at ($(ddeq.south west)+(0,-12mm)$) {};
\node[bias, minimum width=3.5mm, minimum height=3.5mm,
      label={[font=\scriptsize]right:bias corr.\ (KV only)},
      right=24mm of lg1] (lg2) {};
\node[cast, minimum width=3.5mm, minimum height=3.5mm,
      label={[font=\scriptsize]right:cast},
      right=34mm of lg2] (lg3) {};
\end{tikzpicture}%
}
\caption{\method end-to-end pipeline. \textbf{Encode} (top, left to
right) and \textbf{Decode} (bottom, right to left), each
inverse directly below its forward block. Colour marks
applicability: blue is the shared core (weights and KV), orange
dashed is bias correction (KV cache only, ablated in
\Cref{sec:ablation}), and purple is the cast to the Tensor-Core
format. The RHT has no decode block: orthogonal along the
contraction axis, it is absorbed into the matching rotation on the
other MMA operand (ghosted). The MMA is the terminal
consumer, not a codec step. Each block names the subsection that
documents it.}\label{fig:pipeline}
\end{figure}

\subsection{RHT}\label{sec:method-rht}
\emph{RHT.} Partition $x \in \R^n$ into tiles of size
$n_{\mathrm{tile}} = 2^k$ matched to the MMA unit ($128$ on
H100/Blackwell) and apply the RHT~\eqref{eq:rht};
the $O(n_{\mathrm{tile}}\log n_{\mathrm{tile}})$
butterfly folds into the preceding LayerNorm.

\noindent
\emph{Inverse RHT.} None is applied explicitly. The RHT is orthogonal
along the contraction axis, so $W = (WH^\top)H$: the rotation
cancels against the matching rotation on the other MMA operand, and
the decoder never runs an $\mathrm{RHT}^{-1}$ block (ghosted in
\Cref{fig:pipeline}).

\subsection{Rotate (KV only)}\label{sec:method-rotate}
\emph{Rotate.} Optionally rotate by \texttt{none}, \texttt{signs}
($S = \diag(\pm1)$, one bit/coordinate, self-inverse), or \texttt{qjl}
(Haar $S \sim \Unif(\mathrm{O}(n))$); the best choice tracks the
bit-rate (\Cref{sec:results-kv}).

\noindent
\emph{Derotate.} Apply $S^{\top}$. Storage cost and the rotation-dither interaction are detailed in
\Cref{sec:method-kv}.

\subsection{Normalize}\label{sec:method-normalize}
\emph{Normalize.} Rescale to the lattice's calibration radius: for
KV, each (head,\,token) vector by its own norm,
\begin{equation*}
  \tilde{x} \;=\; \alpha\sqrt{n}\,\frac{Sx}{\norm{Sx}},
\end{equation*}
with $\alpha = \alpha(\snr,\Latt)$ the closed-form scale realizing
the target SNR (\Cref{sec:prelim-dither,app:calibration}). Being
deterministic in $x$, this preserves unbiasedness
(\Cref{prop:composition}).

\noindent
\emph{Denormalize.} Multiply back by $\alpha^{-1}$ and the stored
norm.

\subsection{Add dither (KV only)}\label{sec:method-dither}
\emph{Add dither.} Optionally add a fresh $U \sim \Unif(\Vor(\Latt))$
before quantization.

\noindent
\emph{Undither.} Subtract the same $U$ after dequantization. By the
Crypto Lemma the error is then uniform on $\Vor(\Latt)$ and the
reconstruction is strictly inner-product unbiased
(\Cref{cor:ip-unbiased}, \Cref{app:dither_proof}).

\subsection{Quantize}\label{sec:method-quantize}
\emph{Quantize.} Map $\tilde{x}$ (plus dither, if enabled) to its
nearest point $c = \Qfn(\cdot)$ in the integer realization of
$\Latt$; decoding is $O(1)$ for $\Eone,\Dfour,\Atwo$~\citep[Ch.~20]{conway1999sphere}
and nearest-integer rounding for $\Z$.

\noindent
\emph{Dequantize.} Re-embed the stored integer code vector as its
lattice point.

\paragraph{Integer realizations.} The quantization, stripping, Rice
coding, and decoding stages touch the lattice only through (a) its
nearest-neighbor decoder and (b)~the integer code vector it emits.
We are therefore free to pick any \emph{integer} realization of each
lattice, tuned for cheap arithmetic and compact storage. We
use the four families
$\{\Eoneint, \Dfourint, \Atwoint, \Zoneint\}$ of
\Cref{tab:lattices-int}. Two properties motivate these embeddings:

\begin{itemize}[leftmargin=*,itemsep=2pt]
  \item \textbf{8-bit code budget.} After per-vector
    $\alpha$-scaling, the integer coordinates are approximately
    $\mathcal{N}(0,\alpha^2)$, so a
    signed-byte overflow ($|c_i|>127$) is a $127/\alpha$-sigma tail
    event. Even at the top of our sweep ($5$~bps, where $\alpha$ is
    largest) the binding lattice sits $\ge 7\sigma$ from the boundary:
    fewer than $\sim\!10^{-3}$ of the model's $\sim\!7\times 10^{9}$
    coordinates are expected to saturate, and a saturation is a
    harmless clamp to $\pm127$, not a corruption
    (\Cref{app:alpha-calibration}). The raw code vector thus fits
    one signed byte per scalar, matching the storage tile
    of Hopper/Blackwell Tensor Cores and giving \method a natural fallback
    when entropy coding is disabled.
  \item \textbf{Closed-form membership constraints.} Each lattice
    obeys a small set of integer linear constraints
    (parity, coset, sum modulo a power of two). These
    pin a fixed number of bits per code vector, which can
    be stripped from the bitstream before Rice coding without loss.
\end{itemize}

\paragraph{The four embeddings.}
\begin{align*}
\Eoneint &= 2\,\Eone \subset \Z^8, \quad
\Dfourint = \Dfour \subset \Z^4, \quad
\Zoneint = \Z,\\
\Atwoint &= \{(\sqrt{3}\,n_y, n_x) : n_y,n_x \in \Z,\,
n_y{+}n_x \equiv 0 \pmod 2\}.
\end{align*}
The factor-of-two dilation embeds $\Eone$ in $\Z^8$,
clearing the half-integer coset $D_8 {+} \tfrac{1}{2}\mathbf{1}$
of the bare $\Eone$; the $\alpha$-scaling undoes it.
The bare $\Dfour$ is the integer checkerboard lattice
$\{x \in \Z^4 : \sum_i x_i \equiv 0 \pmod 2\}$, which has no
half-integer coset, so $\Dfourint = \Dfour$ already lives in $\Z^4$
and needs no dilation. For
$\Atwoint$ we store the two integer coefficients
$(n_y, n_x)$, folding the $\sqrt{3}$ scaling of the $y$-axis
into $\alpha$ so it never enters the integer
arithmetic. The bare $\Zoneint$ has no nontrivial
membership constraint.

\subsection{Strip}\label{sec:method-lattices}
\emph{Strip.} Strip the bits that lattice membership pins
deterministically (lossless), leaving a compact
symbol stream for Rice coding.

\noindent
\emph{Unstrip.} Reconstruct the pinned bits from the parity
relation and undo the halving.

\paragraph{Membership constraints.} The following equations
characterize membership and form the basis of the bit-stripping
transform.
\begin{itemize}[leftmargin=*,itemsep=2pt]
  \item $\Eoneint$: there exists a \emph{coset bit}
    $c \in \{0,1\}$ such that all coordinates share the same
    parity, $c_i \equiv c \pmod 2$ for $i = 0,\ldots,7$, and the
    halved coordinates satisfy $\sum_{i=0}^{7}(c_i - c)/2 \equiv 0
    \pmod 2$.
  \item $\Dfourint$: $\sum_{i=0}^{3} c_i \equiv 0 \pmod 2$.
  \item $\Atwoint$: $n_y + n_x \equiv 0 \pmod 2$.
  \item $\Zoneint$: no constraint.
\end{itemize}
Each modulo-2 constraint pins one bit of the code vector
deterministically given the rest.

\paragraph{The bit-stripping transform.} For each lattice we
apply, before Rice coding, an invertible map
$\mathrm{Strip}_\Latt : \Z^n \to \Z^{n'}$ that removes the pinned
bits and compacts the remaining symbols:
\begin{center}\small
\begin{tabular}{l >{\raggedright\arraybackslash}p{7.5cm} l c}
\toprule
Lattice & Computation & Output & Saving \\
\midrule
$\Eoneint$ &
  $c = c_0\bmod 2$;
  $s_i = (c_i-c)/2$, $i=0,\ldots,7$;
  $p = (\sum_{i=0}^{6}s_i)\bmod 2$;
  $t = (s_7-p)/2$ &
  $(c,\;s_0,\ldots,s_6,\;t)$ &
  $\mathbf{1.0}$~b/sc \\[3pt]
$\Dfourint$ &
  $p = (c_0+c_1+c_2)\bmod 2$;
  $t = (c_3-p)/2$ &
  $(c_0,\;c_1,\;c_2,\;t)$ &
  $\mathbf{0.25}$~b/sc \\[3pt]
$\Atwoint$ &
  $p = n_x\bmod 2$;
  $t_y = (n_y-p)/2$ &
  $(t_y,\;n_x)$ &
  $\mathbf{0.5}$~b/sc \\[3pt]
$\Zoneint$ &
  (none) &
  $n$ &
  $0$~b/sc \\
\bottomrule
\end{tabular}
\end{center}
The stripped symbols are signed, so the strip's final step maps each
through the \emph{zig-zag} bijection
$\mathrm{zigzag}(n) = 2n$ for $n \ge 0$ and $-2n{-}1$ for $n < 0$,
which keeps small values small and yields the non-negative indices
the Rice coder of \Cref{sec:prelim-rice} expects.
Each transform is bit-for-bit invertible: the decoder reads the
output stream, recovers the halved symbol~$t$, reconstructs the
dropped coordinate ($s_7$ for $\Eoneint$, $c_3$ for $\Dfourint$)
from the parity bit~$p$ computed on the other symbols, rescales
by~$2$, and, for $\Eoneint$, adds back the coset bit.

\paragraph{Rice parameters.} Halving a symbol narrows its
shifted-geometric distribution and lowers its optimal Rice parameter
by one, so a symbol coded at $k_s$ has its halved counterpart best
coded at $k_t = k_s - 1$ (checked at runtime). The three non-trivial
lattices exploit this differently:
\begin{itemize}[leftmargin=*,itemsep=2pt]
  \item $\Eoneint$: the coset bit $c$ is folded into the low bit
    freed by halving via $\mathtt{comb} = 2 \cdot \mathrm{zigzag}(t)
    + c$. Doubling lifts $\mathtt{comb}$ back to the $s$-symbols'
    scale, so $k_{\mathtt{comb}} = k_s$ and all eight symbols share a
    \emph{single} parameter $k_s$: $c$ rides for free in a bit
    Rice would emit anyway. Without the fold, $t$ would need its own
    $k_t = k_s - 1$ and $c$ a separate uncompressed bit.
  \item $\Dfourint$: with no coset bit to fold, the halved symbol
    $t$ keeps its own parameter, so the stream uses \emph{two}
    levels: $k_s$ for $c_0, c_1, c_2$ and $k_t = k_s - 1$ for $t$.
  \item $\Atwoint$: the two coordinates have different spreads (the
    $\sqrt{3}$-scaled $y$-axis is wider), so they too use \emph{two}
    levels: $k_{t_y}$ for $t_y$ and $k_{n_x}$ for $n_x$.
\end{itemize}

\begin{table}[h]
\centering\small
\begin{tabular}{l c c c c}
\toprule
& $\Eoneint$ & $\Dfourint$ &
$\Atwoint$ & $\Zoneint$ \\
\midrule
$n$                       & 8 & 4 & 2 & 1 \\
Embedding                 & $2\Eone \subset \Z^8$ & $\Dfour \subset \Z^4$ & hex.\ $(\sqrt 3 n_y, n_x)$ & $\Z$ \\
Cosets                    & 2 (even/odd) & 2 & 2 & --- \\
Membership constraints    & coset + sum-mod-4 & sum-mod-2 & sum-mod-2 & none \\
Bits stripped per scalar  & $\mathbf{1.00}$ & $\mathbf{0.25}$ & $\mathbf{0.50}$ & $0$ \\
Rice parameters           & 1 ($k_s$, with $c$) & 2 ($k_s, k_t$) & 2 ($k_{t_y}, k_{n_x}$) & 1 \\
\bottomrule
\end{tabular}
\caption{Integer-coordinate realizations of the four lattices
used in \method. ``Bits stripped per scalar'' is the deterministic
information removed by the bit-stripping transform of
\Cref{sec:method-lattices}; these savings are lossless and
applied before Rice coding.}\label{tab:lattices-int}
\end{table}

\paragraph{Effect on the achievable bit-rate.}
The last row of \Cref{tab:lattices-int} is what the entropy of the
\emph{stripped} symbol stream lower-bounds, not the raw code
vector. At a typical $21\,\dB$ SNR (Gaussian
high-rate slope $\sim\!3.7$\,bps), the four lattices reach
empirical Rice rates of $3.74$, $3.77$, $3.81$, and $3.84$\,bps,
within $0.10$\,bps of the high-rate lattice ideal. Removing the
bit-stripping pass would raise the $\Eoneint$ rate by
$1.0$\,bps and the $\Atwoint$ rate by $0.5$\,bps,
\emph{exactly} the deterministic-information overhead Rice
coding cannot otherwise recover.

\begin{remark}[Stripping is rate-optimal, not heuristic]
Stripping does more than delete deterministic bits: it leaves
symbols that are \emph{statistically} near-independent. At high rate
their per-symbol marginal entropy already equals the lattice ideal
$R_D + \tfrac12\log_2(2\pi e\,G(\Latt))$, the rate of an
entropy-coded lattice quantizer, so a memoryless coder such as
Rice is near rate-optimal by construction, with no inter-symbol
redundancy left for a context model. We prove this in
\Cref{app:strip-optimality}.
\end{remark}

\subsection{Rice encode}\label{sec:method-rice}
\emph{Rice encode.} Entropy-code the stripped symbols with the
calibrated Rice code (\Cref{sec:prelim-rice}); on-the-fly bit
accounting, fed by the SNR calibration of
\Cref{app:bps-calibration}, lands the realized rate within
$\sim\!0.01$\,bps of any target. $\Eoneint$ and
$\Zoneint$ each use a single Rice parameter (for $\Eoneint$, the
coset bit $c$ is folded into its remainder, above); $\Dfourint$ uses
two ($k_s$ and $k_t = k_s{-}1$) and $\Atwoint$ two
($k_{t_y}, k_{n_x}$).
\emph{Rice decode.} Unpack the bitstream into symbols.

\subsection{Cast (decode only)}\label{sec:method-cast}
On the 8-/4-bit MMA path, the reconstruction is cast at the matmul
boundary to \fpEightLong/\intEight (Hopper) or \nvfp/\mxfp
(Blackwell); pure-\bff deployments skip it. This block has no
encode counterpart and no inverse: it feeds the terminal MMA.
Format choices and the measured \fpEight-vs-\intEight and
\nvfp-vs-\mxfp trade-offs are in \Cref{sec:results-fp8,sec:results-fp4}.

\subsection{Parameters and how to set them}\label{sec:method-params}

\Cref{tab:params} lists every knob \method exposes, our default
values, and a one-line rationale. Most parameters have
broad sweet spots: Hadamard tile size $128$--$1024$ and Rice
parameter $k \in \{0,1,2\}$ all work equally well at every operating
point we tested, so practitioners typically tune only the
\emph{target bits-per-scalar} and \emph{rotation kind}.

\begin{table}[t]
\centering
\small
\begin{tabular}{l l p{8cm}}
\toprule
Knob & Default & Rationale \\
\midrule
Lattice $\Latt$        & $\Eone$        & Best 8-D granular gain, constant-time decoder, fits MMA tile.\\
RHT tile $n_{\mathrm{tile}}$ & $128$    & Matches H100/Blackwell MMA $K$-dim; auto-shrinks if layer is smaller.\\
Target bps $b$         & $4.0$ (W), $3.0$ (KV) & Sweet spot; $\Delta\ppl \le 0.3$ vs \bff at LLM scale.\\
SNR (derived)          & lookup $b \to \snr$ & Invert the empirical rate curve (\Cref{app:calibration}); $\alpha$ then closed-form.\\
Rotation kind (KV)     & \texttt{qjl}   & Default at $b\!\ge\!2$ bps; switch to \texttt{none} at $b\!\le\!1.6$.\\
Dither (KV)            & off            & Enable when strict per-vector unbiasedness is required.\\
Rice parameter $k$     & $1$            & Auto-tuned per layer if requested.\\
\verb|lm_head| precision   & \bff          & Cheap layer; keep full precision to avoid logit clipping.\\
\bottomrule
\end{tabular}
\caption{\method's complete parameter list, with defaults
benchmarked in \Cref{sec:results,sec:ablation}. Most knobs are
relatively insensitive; only the target bps and (at very low bps)
the rotation kind require tuning per
deployment.}\label{tab:params}
\end{table}

\section{Implementation}\label{sec:implementation}

\subsection{Application to linear weights}\label{sec:method-weights}

The weight path (top of \Cref{fig:pipeline}) is applied once at load
time:
\begin{enumerate}[leftmargin=*,itemsep=2pt]
  \item For each linear layer of shape $(m, n)$, partition $W$ into
    tiles of size $n_{\mathrm{tile}}$ along the input dimension and
    independently RHT each tile.
  \item (If MMA path) cast each tile to \fpEight/\intEight.
  \item Lattice-quantize each tile, bit-strip, and Rice-encode.
  \item Store the resulting integer codes and per-tile scales.
\end{enumerate}
During inference decoding occurs on the fly into the MMA's
input format (\fpEight, \intEight, \nvfp, \mxfp, \bff) just in time
for the matmul; for a fused-Triton implementation the dequant fuses
into the matmul prologue and adds small latency over the bare
MMA. We do \emph{not} quantize the \verb|lm_head| / output projection:
its outputs feed directly into the softmax, where quantization noise
is amplified.

\subsection{Application to the KV cache}\label{sec:method-kv}

The KV path (bottom of \Cref{fig:pipeline}) replaces the \bff
cache tensor with the Rice-coded bitstream plus per-vector norms.
Concretely, for each attention layer we install pre-forward hooks on
\verb|k_proj| and \verb|v_proj| that:
\begin{enumerate}[leftmargin=*,itemsep=2pt]
  \item Receive the projection output in shape $[B, T,
    n_{\mathrm{heads}} \cdot d_{\mathrm{head}}]$.
  \item Reshape to $[B, T, n_{\mathrm{heads}}, d_{\mathrm{head}}]$
    and apply the encoding path (pre-rotation through Rice coding)
    on the last axis.
  \item Store the bitstream as the cache (in our pseudo-quant
    harness we keep an equivalent \bff dequantized tensor for
    simplicity).
\end{enumerate}
At read time the decoder returns a \bff tensor of the original shape
(the ``pseudo-quantization'' regime used for all quality measurements);
a true memory-saving implementation stores only the Rice-coded bitstream
and dequantizes on the fly in a fused attention kernel
(\Cref{sec:results-throughput}).

We hook \emph{pre-RoPE}: since RoPE is a per-position orthogonal
rotation that commutes with $\ell_2$ normalization, pre- and post-RoPE
quantization are statistically equivalent, requiring no modification
to the attention forward function; for GQA/MQA~\citep{llama31} the
hook attaches to each head's projection in isolation.

\paragraph{Choice of pre-rotation.} The three options from
\Cref{sec:prelim-dither} differ in storage cost. A Haar-uniform
$S \sim \Unif(\mathrm{O}(n))$ stores $n^2$ floats per layer (64\,KiB
per layer in fp32 for $n{=}128$, totalling $\sim\!4$\,MiB across the
32 Llama-3.1-8B attention layers). The sign-rotation $S = \diag(\pm
1)$ costs only $n$ bits per layer and is self-inverse; it suffices
when post-RHT activations are approximately exchangeable under coordinate
permutations. We benchmark all three (\texttt{none}, \texttt{signs},
\texttt{qjl}) in \Cref{sec:results-kv}: at $\ge 2$~bps \texttt{qjl}
gives the best PPL; at $\le 1.6$~bps rotation hurts and
\texttt{none} is best; \texttt{signs} is a near-free middle ground.

\paragraph{Composing rotation with dither.} When subtractive
dither is enabled, we draw the rotation $S$ once per
layer at quantize time and the dither $U$ once per forward call;
the two are independent and the composed scheme inherits
isotropic-error covariance from $S$ and strict per-vector
inner-product unbiasedness from $U$ (\Cref{sec:prelim-dither}).

\subsection{Decoding a variable-length code on the GPU}\label{sec:method-decode}

\paragraph{The challenge.} \method's rate gain comes from the
variable-length Rice code over lattice indices (\Cref{sec:method-rice}):
each index costs a data-dependent number of bits, so the bit position of
symbol $j$ depends on every preceding symbol. This sequential dependency
is the central obstacle to GPU decoding: unlike a fixed-width format
(\intEight or \intFour), a Rice stream cannot be random-accessed or
bulk-loaded, and a naive decoder is a single serial scan.

\paragraph{Approaches.} Four strategies appear in the literature:
\emph{(i)~bit-serial decode}~\citep{wiegand2003h264}, which parallelizes
across independent slices but stays serial within each stream;
\emph{(ii)~fixed-chunk multi-pass synchronization}~\citep{weissenberger2018huffman},
which decodes chunks provisionally in parallel then recovers codeword
boundaries with synchronization passes;
\emph{(iii)~offset-indexed one-pass}, which stores an explicit start
bit-offset per sub-stream so each thread decodes independently in a
single pass; and \emph{(iv)~rANS}~\citep{duda2013ans}, which admits
$N$-way SIMD decoding at near-identical rate but requires changing the
codec.

\paragraph{Our choice.} \method uses the \textbf{offset-indexed one-pass}
decoder. The encoder groups lattice symbols into sub-streams of $S$
symbols and emits, per sub-stream, a 32-bit start bit-offset and the
per-stream Rice parameter $k$. The decoder launches one thread per
sub-stream, each seeking to its offset and decoding $S$ symbols with
the inverse bit-strip and per-tile MMA cast fused inline, directly
into the \bff/\intEight/\fpEight/\nvfp input format. This design is
(i)~single-pass and branch-light (no seam synchronization);
(ii)~cheap in metadata, with the offset$+k$ table at $32/S$ bits per
symbol ($\lesssim\!1\%$ of the $\sim\!4$-bps payload at $S{=}512$);
and (iii)~composable with bit-stripping and the MMA cast in a single
pass. The trade-off is that $S$ couples compression against parallel
occupancy (larger $S$: fewer threads); we use $S{=}512$.

\subsection{Reference implementation}\label{sec:method-refimpl}

The \method reference implementation is $\sim\!3$\,k lines of Python plus
the CUDA implementation of the decoder. A single post-training pass over the loaded \bff
model installs the KV hooks and quantizes-in-place the linear weights; it
is parallelizable per layer and finishes in $\sim\!30$~s for Llama-3.1-8B
on a single H100.

\paragraph{Calibration: SNR-to-bps lookup table.} For a Gaussian input
$x \sim \mathcal{N}(0, \sigma^2 I_n)$ with lattice scaled so that the
per-scalar quantization noise has variance $\sigma_q^2$, the per-scalar
SNR is $\snr = 10\log_{10}(\sigma^2/\sigma_q^2)$, and the Rice-coded
bit-rate is a monotone function of $\snr$. We pre-build a calibration
table mapping SNR to empirical Gaussian bps on $\sim\!10^5$ iid-Gaussian
vectors per operating point and cache it to disk. At inference time we look up the SNR whose realized rate lands within
$\sim\!0.01$\,bps of any requested target (full procedure:
\Cref{app:calibration}), enabling arbitrary fractional bit-rates that
fixed-rate codebooks cannot match.

\paragraph{Hyperparameter auto-tuning.} For a given target rate, the
implementation looks up the lattice SNR via interpolation in the cached
table and applies it uniformly across all quantized tensors. Per-layer
bit allocation is supported but disabled by default; we leave its study
as future work (\Cref{sec:conclusion}).

%% file: sections/results.tex
\section{Experiments}\label{sec:results}

We evaluate \method in three stages: characterizing the method on its
own (\Cref{sec:results-weights}--\Cref{sec:results-ltx}), measuring
deployment cost (\Cref{sec:results-throughput}), and comparing against
prior codecs (\Cref{sec:results-comparison}). Quality numbers are exact
pseudo-quantization PPL/quality measurements; throughput and memory are
measured end-to-end on an H100.

\paragraph{Setup.} The LLM weight/KV experiments use
\texttt{Llama-3.1-8B-Instruct} \citep{llama31} evaluated on the WikiText-2 raw
test split \citep{merity2017wikitext} using 141 non-overlapping windows of
2048 tokens (\bff baseline PPL $7.1606$); the KV comparison against OCTOPUS
additionally uses \texttt{Qwen2.5-7B-Instruct-1M} at $4096$-token context
(\Cref{sec:cmp-octopus}). The video experiment uses \texttt{LTX-2-19B}
\citep{ltx2} on a 32-prompt suite at 512$\times$320 resolution and 49 frames
(Stage~1 only). All experiments are post-training and use no fine-tuning or
calibration data; the calibration that \emph{is} required is the synthetic
SNR$\leftrightarrow$bps lookup table of \Cref{app:bps-calibration}, computed
once.


\paragraph{Part I: Method characterization.}
\subsection{Weight-only quantization}
\label{sec:results-weights}

\method quantizes every \texttt{nn.Linear} weight (except \texttt{lm\_head})
with per-tile RHT, per-block $\alpha$-scaling, lattice quantization,
bit-stripping, and Rice coding of the integer codes
(\Cref{sec:method-steps,sec:method-lattices}); no \fpEight cast and no calibration
data. We sweep target rates from $3.0$ to $5.0$ bps for all four lattices
($\Eone,\Dfour,\Atwo,\Z$). Because the Rice code is variable-length, the rate
knob is \emph{continuous}: a single $\alpha$ per RHT tile lands the
realized rate within $0.01$\,bps of any requested target.

\begin{figure}[!ht]
\centering
\includegraphics[width=\linewidth]{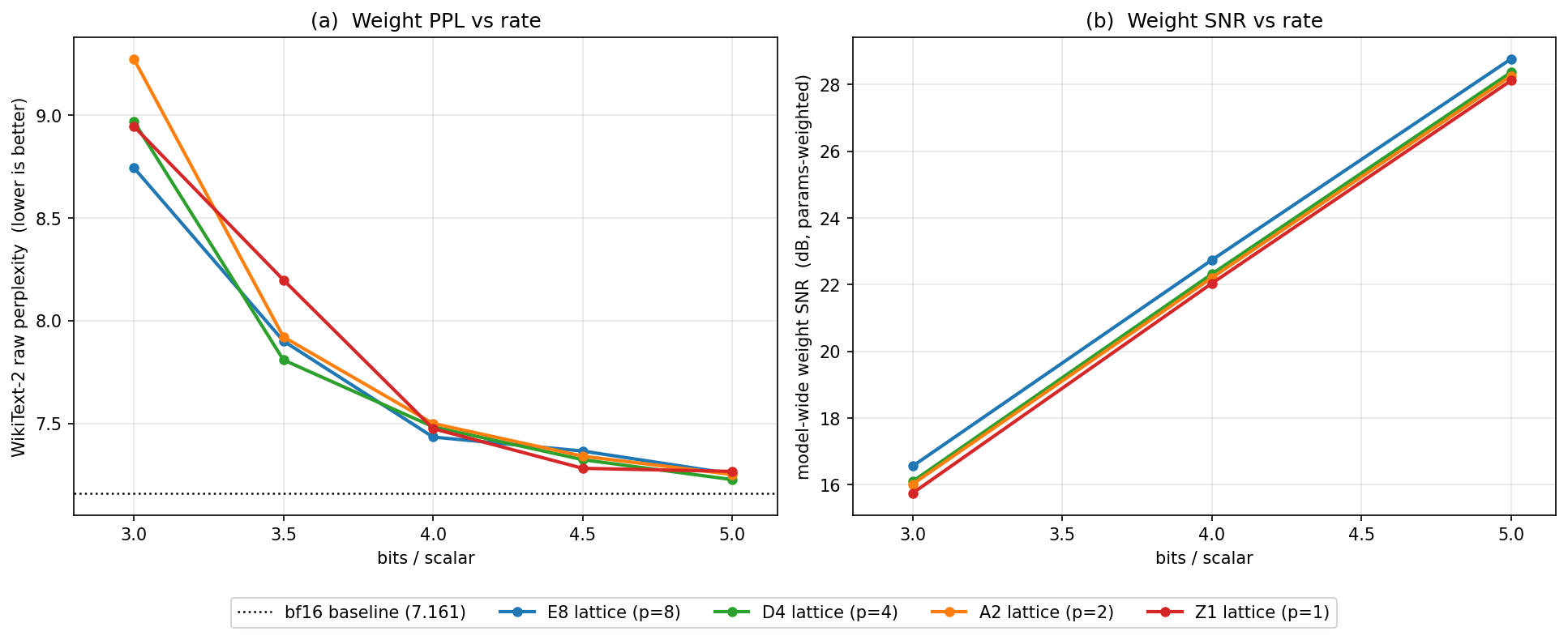}
\caption{Weight quantization on Llama-3.1-8B at matched bps across the
four lattices ($\Eone,\Dfour,\Atwo,\Z$). \textbf{(a)} WikiText-2 PPL vs
rate (\bff baseline $7.161$): at $b\le3.25$ the higher-dimensional
$\Eone$ wins by a clear margin, while at $b\ge4.5$ the four lattices
cluster within $0.05$ PPL, below the run-to-run eval-noise floor.
\textbf{(b)} Model-wide weight SNR (dB, params-weighted) vs rate: the
on-model SNR matches the iid-Gaussian calibration target to within
$\pm0.02$~dB and orders exactly as $\Eone>\Dfour>\Atwo>\Z$ at every
rate.}\label{fig:weight-lattice}
\end{figure}

\paragraph{Lattice ordering follows the Voronoi second moment.} The
per-lattice weight SNR (params-weighted over the $224$ quantized layers) is
constant to within $\pm0.05$~dB across layers and orders exactly as the
textbook normalized second moments $G(\Latt)$
(\Cref{fig:weight-lattice}b and \Cref{tab:lattice-constants}): $\Eone>\Dfour>\Atwo>\Z$. The
advantage of the best lattice over the worst grows with rate
($\Eone-\Z$: $0.81$~dB at 3~bps to $0.65$~dB at 5~bps), reflecting the
high-rate regime where granular gain dominates. In practice we pick $\Eone$
below $3.25$ bps, $\Eone/\Dfour$ in the $3.5$--$4.0$ range, and any lattice
above $4.25$ bps (where the choice is below the PPL noise floor, so kernel
simplicity, $\Z$ scalar, $\Atwo$ 2-D, $\Dfour$ 4-D, $\Eone$ 8-D, decides).

\paragraph{SNR is the sufficient statistic.} Across all four lattices and all
rates, model PPL is a single monotone function of weight SNR, as the linearity
theorem~\citep{malinovskii2025higgs} predicts: it reduces global perplexity
damage to a sum of per-layer mean-squared errors, so equal weight SNR implies
equal expected PPL hit regardless of error shape. We therefore calibrate on one
synthetic SNR$\leftrightarrow$bps table and let PPL fall out for free rather
than running end-to-end PPL per configuration; \Cref{sec:cmp-higgs} shows this
collapse holds across schemes too.

\subsection{KV-cache-only quantization and bias correction}
\label{sec:results-kv}

We benchmark the \method KV path with \bff weights, sweeping target bps from
$1.5$ to $4.0$ and dequantizing the cache to \bff before attention.
\Cref{fig:kv-bps} summarizes the rate-quality behavior.

\begin{figure}[!ht]
\centering
\includegraphics[width=\linewidth]{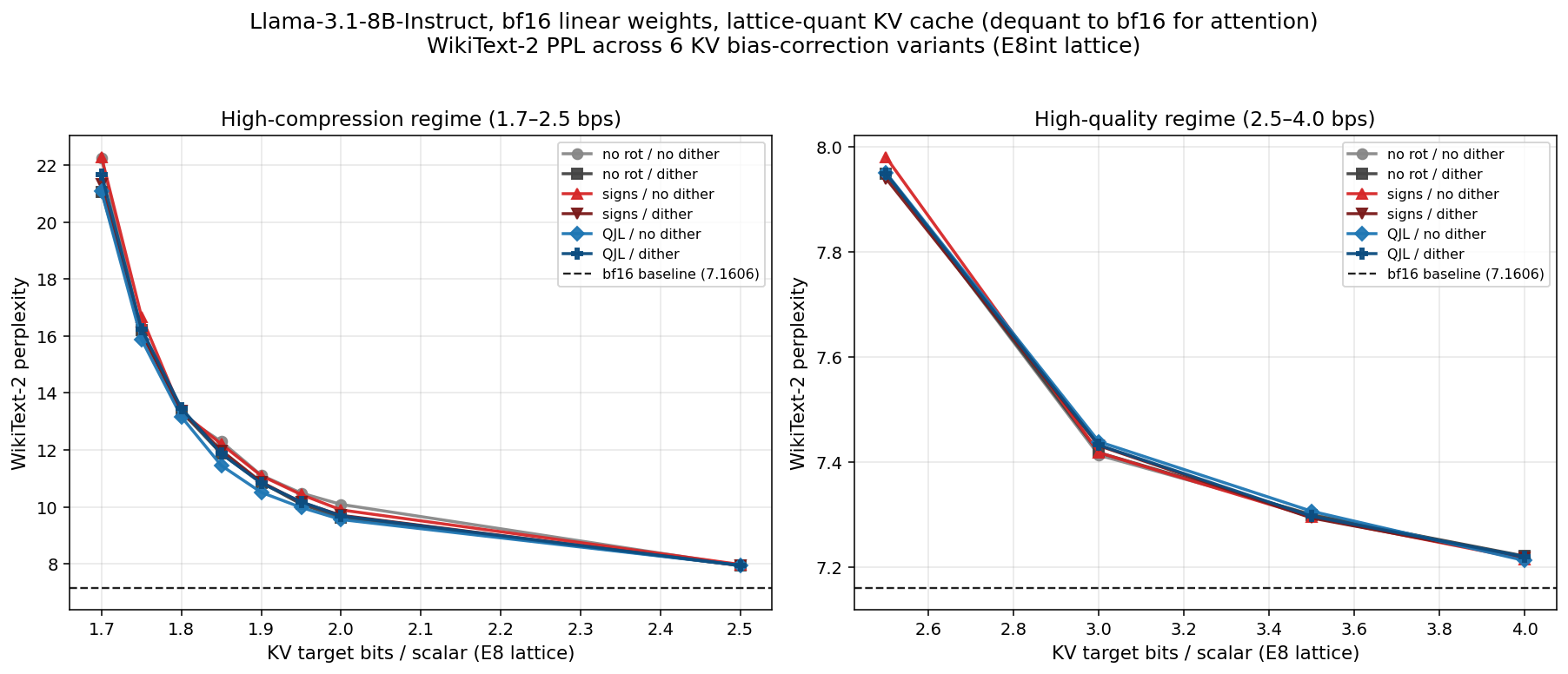}
\caption{KV-only \method on Llama-3.1-8B with \bff weights, shown
over the two working regimes. (Left) High-compression regime
($1.7$--$2.5$~bps): QJL/signs rotation pulls ahead of plain
\texttt{none} as the rate falls, reaching $\sim 0.5$~PPL at 2.0~bps.
(Right) High-quality regime ($2.5$--$4.0$~bps): all six
bias-correction variants collapse onto the \bff baseline, within
$0.04$~PPL of one another at $b\ge 2.5$.}
\label{fig:kv-bps}
\end{figure}

\paragraph{Two operating regimes.}
\Cref{fig:kv-bps} splits cleanly into two working regimes,
summarized in \Cref{tab:kv-regimes}:

\begin{table}[!ht]
\centering
\small
\begin{tabular}{l l l l}
\toprule
Regime & bps range & Best variant & $\Delta$ PPL vs.\ \bff \\
\midrule
High-quality      & $\ge 2.5$ & \texttt{none} (variants tied) & $0.05$--$0.79$ \\
High-compression  & $1.7$--$2.5$ & \texttt{qjl} / \texttt{signs} & $0.79$--$13.9$ \\
\bottomrule
\end{tabular}
\caption{The two working regimes of KV-only lattice quantization.
In the high-quality regime ($b\ge 2.5$) every bias-correction
variant is within $0.04$~PPL, so the cheapest (\texttt{none}) is
fine; in the high-compression regime ($1.7$--$2.5$~bps) QJL or the
$\pm1$ \texttt{signs} rotation pulls ahead, by $\sim 0.5$~PPL at
2.0~bps.}\label{tab:kv-regimes}
\end{table}

\paragraph{High-compression floor.} The marginal cost steepens
sharply toward the bottom of the high-compression regime: each
$0.05$-bps step below $b\approx 1.8$ roughly doubles the added PPL
($+2.9 \to +6.1$ PPL per $0.05$-bps for \texttt{none} between
1.8~bps and 1.7~bps). The lattice cell radius grows faster than the
per-vector signal at this rate, and the linearity-of-attention
argument that makes the high-quality regime forgiving breaks down,
setting the $\sim 1.7$-bps practical floor for data-free KV
quantization. (\Cref{sec:cmp-octopus} shows that a small residual
window of recent tokens largely defers this floor.)

\paragraph{Sweet spot at 3.0~bps.} \method achieves $\Delta\ppl =
+0.25$ at 3.0~bps with an $81\%$ KV-cache memory reduction. This
is, to our knowledge, the \emph{best published KV-only quality at
3 bps on Llama-3.1-8B without calibration}.

\paragraph{Bias-correction variants.} \method's KV path supports two
bias-correction variants (\Cref{sec:method-kv}): a per-layer random
\emph{rotation} (none, a cheap $\pm1$ \emph{signs} diagonal, or a full Haar
\emph{QJL} matrix) to reduce within-vector anisotropy, and subtractive
\emph{dither} to make the inner-product error strictly unbiased per cached
vector (provable via the Schuchman conditions; \Cref{app:dither_proof}).
\Cref{tab:kv-bias} sweeps all six (rotation $\times$ dither) combinations
KV-only at $\Eone$, 2.0~bps (\bff weights), the high-compression
regime where the choice actually moves PPL.

\begin{table}[!ht]
\centering\small
\begin{tabular}{l c c c c c}
\toprule
variant & rotation & dither & PPL & $\Delta$PPL vs.\ \bff & per-vec bias \\
\midrule
\texttt{none}        & --   & off & $10.089$ & $+2.928$ & $-2.2\!\times\!10^{-2}$ \\
\texttt{dither}      & --   & on  & $9.622$  & $+2.462$ & $\approx 0$ \\
\texttt{signs}       & $\pm1$ & off & $9.896$  & $+2.735$ & $-4.0\!\times\!10^{-2}$ \\
\texttt{signs+dith}  & $\pm1$ & on  & $9.706$  & $+2.546$ & $\approx 0$ \\
\texttt{qjl}         & Haar & off & $\mathbf{9.557}$ & $\mathbf{+2.397}$ & $+1.9\!\times\!10^{-2}$ \\
\texttt{qjl+dith}    & Haar & on  & $9.685$  & $+2.524$ & $\approx 0$ \\
\bottomrule
\end{tabular}
\caption{Bias-correction variants for \method KV on Llama-3.1-8B
(KV-only, $\Eone$ at 2.0~bps, \bff weights; $\Delta$PPL is over the
\bff baseline at $7.161$). Unlike at $4$ bps, the choice matters in this
high-compression regime: the variants span $\sim0.53$ PPL. \texttt{qjl}
(no dither) is the PPL winner, and both a rotation (\texttt{qjl}/\texttt{signs})
and \texttt{dither} independently improve on plain \texttt{none}; dither
additionally buys exact per-vector unbiasedness ($\approx 0$ bias), and the
$\pm1$ \texttt{signs} rotation tracks QJL at $1/128$ the stored
memory.}\label{tab:kv-bias}
\end{table}

At 4~bps this choice is in the noise (all six within $0.014$ PPL); the
dominant gain there is the \bff$\to$4-bit quantization itself. The spread
opens at lower rates, where rotation pulls ahead by $\sim0.5$ PPL
(\Cref{tab:kv-regimes}) and dither matters most for long-context workloads
where per-vector bias compounds across thousands of tokens. We default to
\texttt{qjl} (or \texttt{signs} when rotation storage is a concern), adding
dither only when provable unbiasedness is required.

\subsection{Full-model quantization at 8-bit MMA precision}
\label{sec:results-fp8}

We compose all four \method components: weights and KV at 4~bps,
with a per-tile \fpEight or \intEight cast for the MMA and optional bias
correction. \Cref{tab:fp8-summary} summarizes the end-to-end quality.

\begin{table}[!ht]
\centering\small
\begin{tabular}{l l c c}
\toprule
MMA cast & Path & PPL & $\Delta$PPL \\
\midrule
\bff      & ---            & $7.161$          & --- \\
\fpEight  & weights-only   & $7.535$          & $+0.37$ \\
\fpEight  & weights$+$KV   & $7.644$          & $+0.48$ \\
\intEight & weights-only   & $7.433$          & $+0.27$ \\
\intEight & weights$+$KV   & $\mathbf{7.503}$ & $+0.34$ \\
\bottomrule
\end{tabular}
\caption{Llama-3.1-8B end-to-end \method at 4~bps with 8-bit MMA.
Adding the KV-cache pipeline on top of the weights$+$8-bit path costs
only $+0.11$ PPL (\fpEight) or $+0.05$ PPL (\intEight). KV bias-correction
choice (\texttt{none}/\texttt{dither}/\texttt{signs}) moves PPL by
$<0.01$ (\texttt{none} shown).}\label{tab:fp8-summary}
\end{table}

\paragraph{INT8 wins on post-RHT data.} \intEight beats
\fpEightLong by $\sim 0.10$ PPL at matched precision
(\Cref{tab:fp8-summary}). Post-RHT, post-lattice tensors have light
tails (almost bounded), so \intEight's $256$ equally spaced levels
are more useful than \fpEight's logarithmic spacing of $200$ effective
levels plus $56$ wasted on the tails. The order flips on raw
activations where outliers dominate; the lattice path renders that
trade-off in favor of \intEight.

\subsection{Full-model quantization at 4-bit MMA precision}
\label{sec:results-fp4}

The Blackwell generation exposes \nvfp (16-element \fpEight-scaled
blocks) and \mxfp (32-element power-of-2-scaled blocks). We feed
\method's lattice output into both formats at 4~bps and 3~bps;
\Cref{tab:fp4-summary} summarizes the results.

\begin{table}[!ht]
\centering\small
\begin{tabular}{l l l c c}
\toprule
Format & Path & Bias & PPL & $\Delta$PPL \\
\midrule
\multicolumn{5}{l}{\emph{4~bps \method}} \\
\nvfp & weights-only & \texttt{none}   & $8.43$        & $+1.27$ \\
\nvfp & weights$+$KV & \texttt{none}   & $9.29$        & $+2.13$ \\
\nvfp & weights$+$KV & \texttt{dither} & $\mathbf{9.24}$ & $+2.08$ \\
\mxfp & weights-only & \texttt{none}   & $9.46$        & $+2.30$ \\
\mxfp & weights$+$KV & \texttt{none}   & $\sim 18$     & $\sim+11$ \\
\mxfp & weights$+$KV & \texttt{dither} & $13.61$       & $+6.45$ \\
\midrule
\multicolumn{5}{l}{\emph{3~bps \method}} \\
\nvfp & weights$+$KV & \texttt{dither} & $\mathbf{12.81}$ & $+5.65$ \\
\mxfp & weights$+$KV & \texttt{dither} & $25.42$       & $+18.26$ \\
\bottomrule
\end{tabular}
\caption{Blackwell \fpFour path at \method lattice bases of $4$ and $3$
bps ($\Delta$PPL vs.\ \bff{} $7.161$). \nvfp${}+{}$dither is the only
\mxfp-class configuration that survives the KV cache; \mxfp's \Eeightmzero
scale loses too much dynamic range to handle KV tails.}
\label{tab:fp4-summary}
\end{table}

Dither has a small effect on \nvfp ($-0.05$ PPL, the
\texttt{none}$\to$\texttt{dither} weights$+$KV rows) but a large
positive effect on \mxfp ($-0.33$ PPL at 4~bps); at 3~bps the
dither rescue on \mxfp jumps to $-5.89$ PPL (model goes from broken
to borderline). The pattern is consistent with the dither role being
\emph{more important at higher quantization noise}, which both
lower bps and the coarser \mxfp grid produce.

\subsection{Beyond LLMs: LTX-2-19B video DiT}
\label{sec:results-ltx}

We apply \method to LTX-2-19B~\citep{ltx2}, a 19B-parameter diffusion
transformer (DiT) for text-to-video synthesis with $1370$ linear
layers, quantizing all weights at 4~bps with \fpEight or \intEight MMA
and leaving the rest of the pipeline (Gemma-3-12B text encoder, VAE,
scheduler) at \bff.

\begin{figure}[!ht]
\centering
\includegraphics[width=0.95\linewidth]{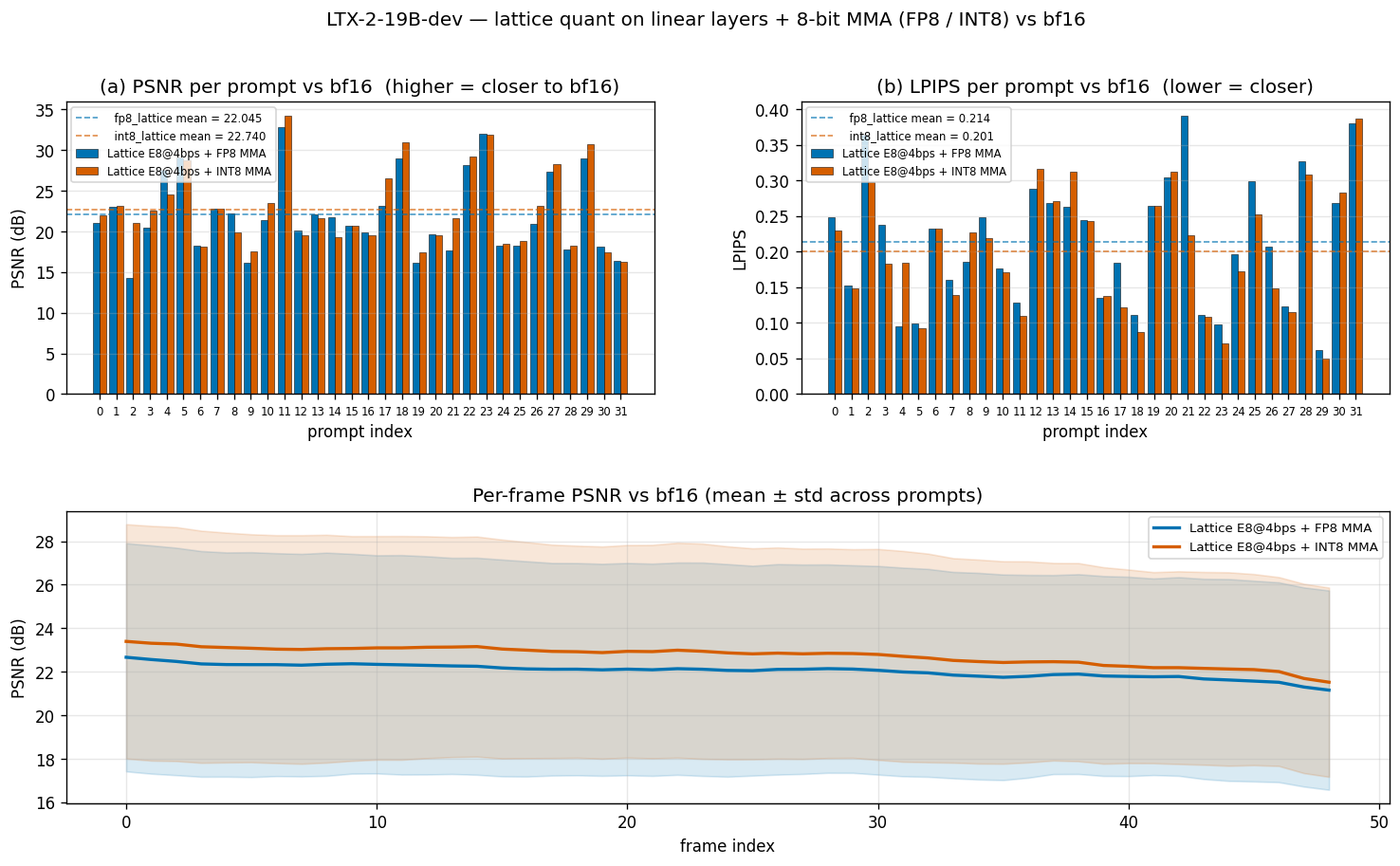}
\caption{Per-prompt PSNR (a) and LPIPS (b), and per-frame PSNR
(mean $\pm$ std) for \method on LTX-2-19B versus the \bff baseline,
on a 32-prompt suite at 512$\times$320, 49 frames. \intEight MMA edges
\fpEight on both PSNR and LPIPS.}\label{fig:ltx-quality}
\end{figure}

\begin{figure}[!ht]
\centering
\includegraphics[width=\linewidth]{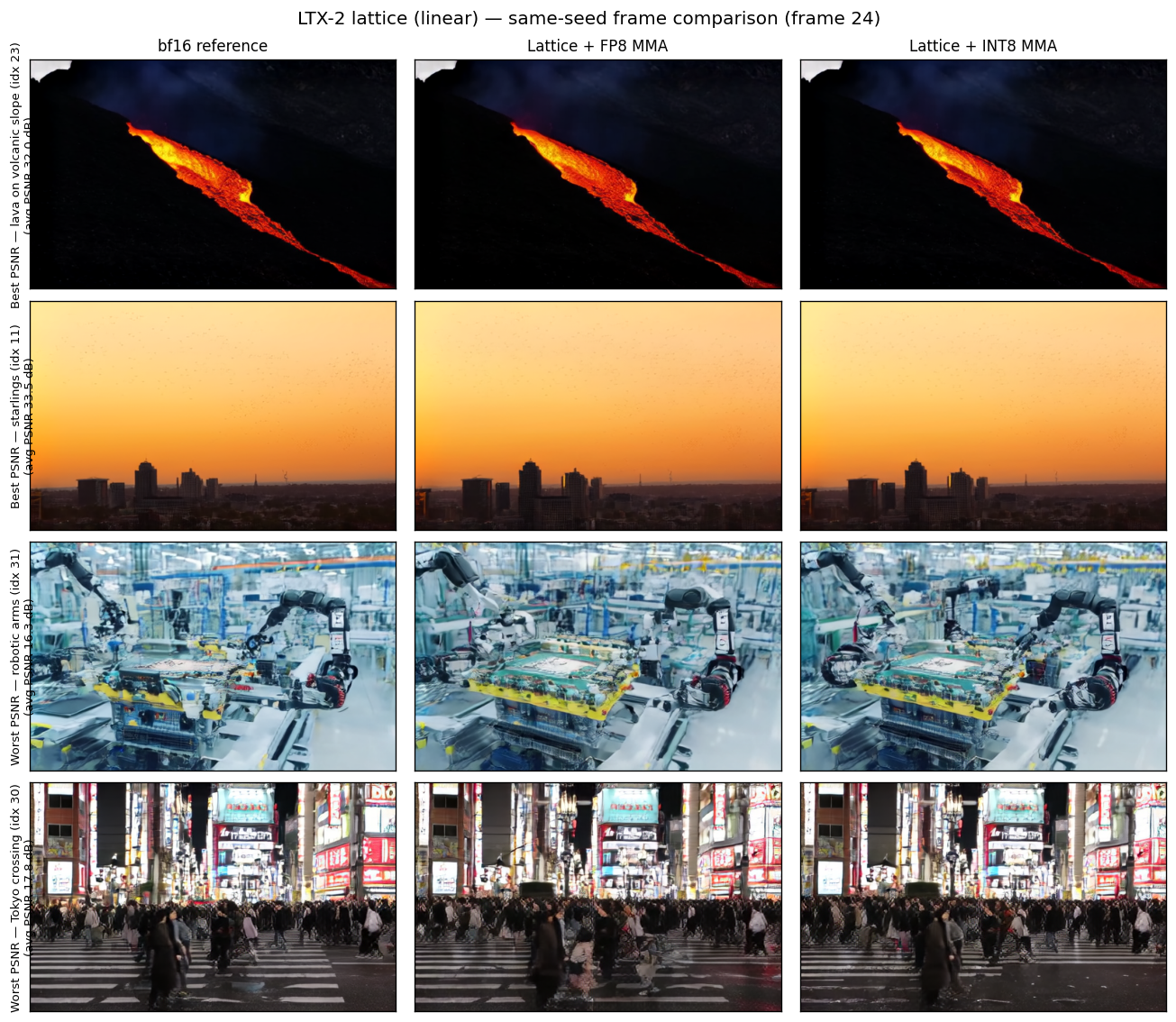}
\caption{Sample frames from \method on LTX-2: \bff baseline (top),
\intEight + lattice (middle), and per-frame error map (bottom). No
visible artefacts; per-frame PSNR is flat across the 49-frame
window.}\label{fig:ltx-frames}
\end{figure}

\begin{table}[!ht]
\centering
\small
\begin{tabular}{l c c c}
\toprule
Config             & PSNR (dB) $\uparrow$ & SSIM $\uparrow$ & LPIPS $\downarrow$ \\
\midrule
\bff baseline      & ($\infty$)           & $1.000$         & $0$ \\
\method + \fpEight MMA  & $22.04$              & $0.8068$        & $0.2144$ \\
\method + \intEight MMA & $\mathbf{22.74}$     & $\mathbf{0.8172}$ & $\mathbf{0.2008}$ \\
\bottomrule
\end{tabular}
\caption{LTX-2-19B quality under \method, 32-prompt evaluation. The
\intEight MMA path is better than \fpEight on every metric at identical
4~bps.}\label{tab:ltx-quality}
\end{table}

\intEight MMA achieves PSNR $22.74$~dB, SSIM $0.8172$, and LPIPS $0.2008$
at 4~bps, improving on the \fpEight variant on every metric. The
\intEight-beats-\fpEight result from the LLM experiments (\Cref{sec:results-fp8})
thus replicates on a much larger model in a qualitatively different
domain. Weight memory shrinks $35.16 \to 9.5$~GiB ($3.7\times$);
generation wall-clock is slightly slower (\fpEight $254.0$~s vs.\ \bff
$209.7$~s) because the pseudo-quantization harness adds a decode pass
without exercising MMA acceleration. As on the LLM
(\Cref{sec:results-throughput}), \method's hardware win is the
$3.7\times$ weight-memory reduction, not wall-clock.

\paragraph{Per-frame analysis.} Per-frame PSNR is essentially constant
across the 49-frame window (\Cref{fig:ltx-frames}): quantization noise
does not compound through the DiT's temporal conditioning. The low
absolute PSNR (22--23~dB) reflects the natural posterior divergence of
any diffusion model under perturbation, not visible artefacts.

\paragraph{Part II: Deployment.}
\subsection{End-to-end throughput and memory}\label{sec:results-throughput}

\Cref{tab:throughput} reports end-to-end Llama-3.1-8B throughput and
resident memory on a single H100 (decode: autoregressive $M{=}1$; prefill:
one $2048$-token forward), together with the weight and KV compression.
\method compresses the linear weights $3.9\times$, cutting full-model
resident memory $\sim\!2.8\times$ ($14.96\!\to\!5.29$~GiB). The
full-model factor trails the $3.9\times$ weight factor because the token
embeddings and \texttt{lm\_head} are kept in \bff: that $5.29$~GiB is
$3.32$~GiB of compressed linear weights plus $1.96$~GiB of \bff
embeddings/head. For much larger models we expect this difference between the compression factors to shrink significantly.

For the KV cache the \emph{resident} column (measured before any generation
begins) does not change, because the cache is empty at model-load time.
The KV savings materialize during generation: E8 lattice codes are stored
as a variable-length Rice-coded bitstream plus a \texttt{float16} per-vector
L2 norm, yielding $\sim\!3.8\times$ actual GPU memory reduction per cached
token ($0.516$~bytes/scalar vs.\ $2$~bytes/scalar for \bff; saving
${\sim}0.09$~GiB per $1{,}024$ tokens and ${\sim}2.9$~GiB per
$32{,}768$-token context). The gap between the $3.8\times$ actual figure and the
$4\times$ theoretic value is metadata overhead
(bit-offset table and Rice-$k$ entry per stream).

Neither path is a throughput win on this hardware. For the weight path,
a warp-specialized fused decode$+$GEMV kernel reduces per-layer
memory traffic from ${\sim}4.5$~B/scalar (read bitstream $+$ write scratch
$+$ read scratch by cuBLAS) to ${\sim}2.5$~B/scalar (bitstream read only,
$x$ in shared memory), yielding ${\sim}1.5\times$ decode speedup on the
weight-quantized path. For the KV path, the past bitstream is maintained
as a single contiguous tensor and decoded with one kernel call per role per
layer; the remaining overhead is the $O(T)$ per-step re-decode of all past
tokens and the QJL inverse rotation ($O(TD^2)$ per layer). Eliminating these requires the kernel directions in
\Cref{sec:conclusion}.


\begin{table}[!ht]
\centering\small
\begin{tabular}{l c c c c c}
\toprule
Config & prefill (tok/s) & decode (tok/s) & resident (GiB) & weight cmp. & KV cmp. \\
\midrule
\bff baseline & $16{,}505$ & $51.8$ & $14.96$ & -- & -- \\
\method W (4 bps) & $3{,}915$  & $7.8$  & $5.29$ & $3.9\times$ & -- \\
\method KV (4 bps)  & $16{,}261$ & $10.8$ & $14.99$ & -- & $3.79\times$ \\
\method W$+$KV & $4{,}655$  & $5.4$  & $5.69$ & $3.9\times$ & $3.79\times$ \\
\bottomrule
\end{tabular}
\caption{End-to-end Llama-3.1-8B-Instruct on one H100, \bff base, all at
$4$~bps. \method compresses the linear weights $3.9\times$ (full-model
resident $2.8\times$, $14.96\!\to\!5.29$~GiB) and stores the KV cache as
variable-length Rice bitstreams, at a throughput cost from
the per-forward weight decode and the $O(T)$ per-step KV
re-decode.}\label{tab:throughput}
\end{table}


\paragraph{Part III: Comparison to prior work.}
\subsection{Comparison to prior quantization schemes}
\label{sec:results-comparison}

We compare \method against the strongest published codecs in each
setting: HIGGS for weights (\Cref{sec:cmp-higgs}) and
TurboQuant/OCTOPUS for the KV cache (\Cref{sec:cmp-octopus}).

\subsubsection{Weights: \method vs.\ HIGGS}
\label{sec:cmp-higgs}

We compare the \method weight path against HIGGS
\citep{malinovskii2025higgs} at matched bit-rates from $3$ to $5$ bps.
HIGGS runs at its native fixed rates (3, 3.5, 4, 4.5, 5~bps for
$p{=}2$; 3, 4, 5~bps for $p{=}1$), with half-integer points using a
Lloyd-trained codebook of size $\sqrt{2}\cdot 2^b$
\citep{pages2003optimal}; \method runs continuously via Rice.
\Cref{fig:weights-bps} shows the headline result: every \method
lattice beats HIGGS-$p2$ at every rate.

\begin{figure}[!ht]
\centering
\includegraphics[width=0.9\linewidth]{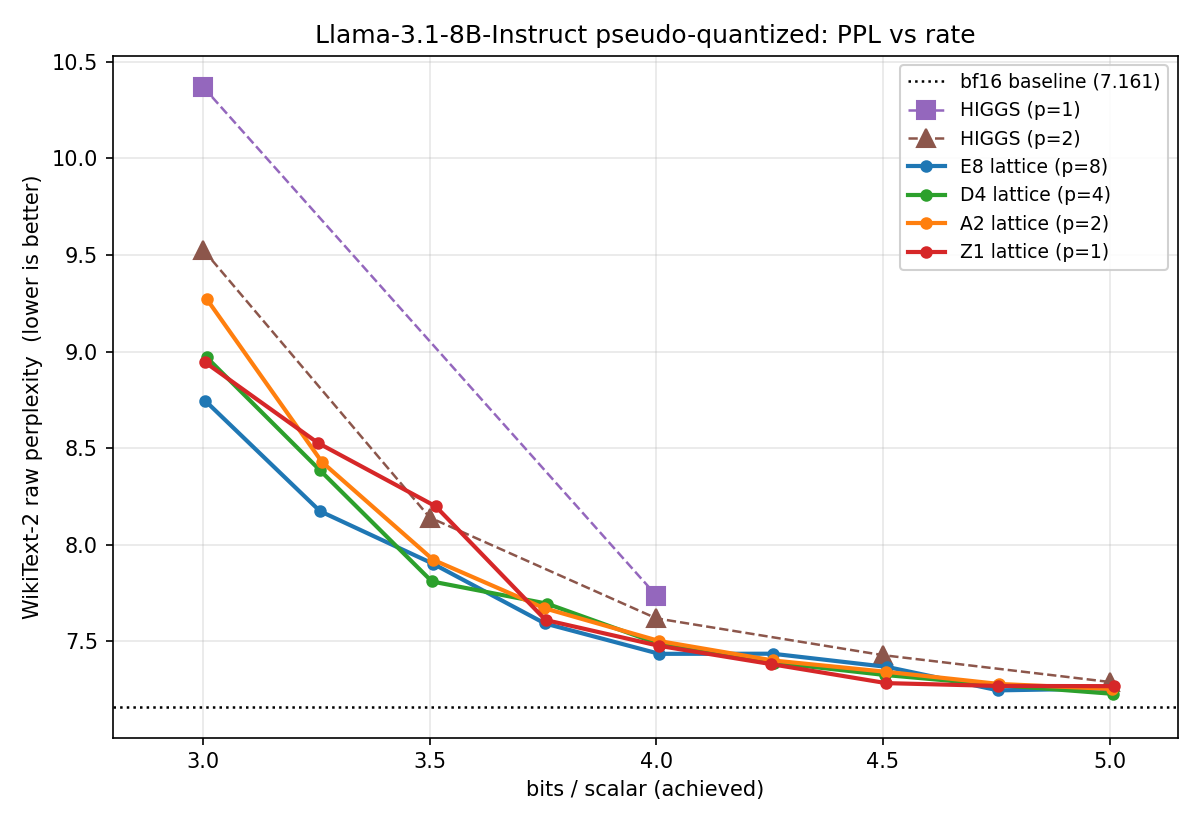}
\caption{Llama-3.1-8B WikiText-2 PPL versus bits per scalar for
\method (lattices $\Eone$, $\Dfour$, $\Atwo$, scalar $\Z$) and
HIGGS ($p\in\{1, 2\}$). \method significantly outperforms HIGGS at every bps.
The four lattices cluster together at $b \ge 4.25$ because their
asymptotic $G(\Latt)$ values are within $0.5$~dB.}
\label{fig:weights-bps}
\end{figure}

\paragraph{A dimension-matched comparison.} HIGGS-$p2$ is a
Lloyd-optimal \emph{two-dimensional} codebook, so the fair
head-to-head is against \method's two-dimensional lattice $\Atwo$
(not the higher-dimensional $\Eone$, which we return to below). Even
at matched dimension, $\Atwo$ wins at every rate:
\begin{center}\small
\begin{tabular}{l c c c c c}
\toprule
bps             & 3.0  & 3.5  & 4.0  & 4.5  & 5.0 \\
\midrule
HIGGS-$p2$ PPL        & 9.527 & 8.140 & 7.618 & 7.427 & 7.288 \\
\method ($\Atwo$) PPL & \textbf{9.273} & \textbf{7.921} & \textbf{7.500}
& \textbf{7.341} & \textbf{7.252} \\
$\Delta$ PPL          & $-0.25$ & $-0.22$ & $-0.12$ & $-0.09$ & $-0.04$ \\
\bottomrule
\end{tabular}
\end{center}

\paragraph{Where the gap comes from.} HIGGS uses a \emph{finite}
codebook of $N = 2^{pb}$ codewords at a fixed $\log_2 N$ bits per
index; \method uses an \emph{unbounded} integer lattice with a
variable-length Rice code, allowing the codebook to extend to
infinity at finite expected rate~\citep{zamir2014lattice}.
Converting rate slack to SNR at the high-rate Gaussian slope
($6.02$~dB/bps) splits the $\Atwo$-vs-HIGGS-$p2$ gap into two
separable pieces (\Cref{tab:gap-decomp}):
\begin{itemize}[leftmargin=*,itemsep=2pt]
  \item \emph{Index-entropy piece.} HIGGS spends $\log_2 N$ bits per
        index even though its index histogram has lower entropy;
        Rice coding recovers this slack (\Cref{fig:entropy-gain}).
        It dominates at low rate ($0.61$~dB at 3~bps) but shrinks
        as the Lloyd histogram becomes more uniform ($0.18$~dB at
        5~bps).
  \item \emph{Unbounded-codebook piece.} A finite codebook must
        stretch its outermost cells to cover the Gaussian tails; the
        lattice has no boundary cell, so a tail outlier merely
        produces a larger-integer code that costs proportionally
        more bits. This residual \emph{grows} with rate
        ($0.13 \to 1.79$~dB from 3~bps to 5~bps) and is
        realizable only because variable-length coding lets the
        codebook be unbounded in the first place.
\end{itemize}

\begin{figure}[!ht]
\centering
\begin{subfigure}{0.49\linewidth}
  \includegraphics[width=\linewidth]{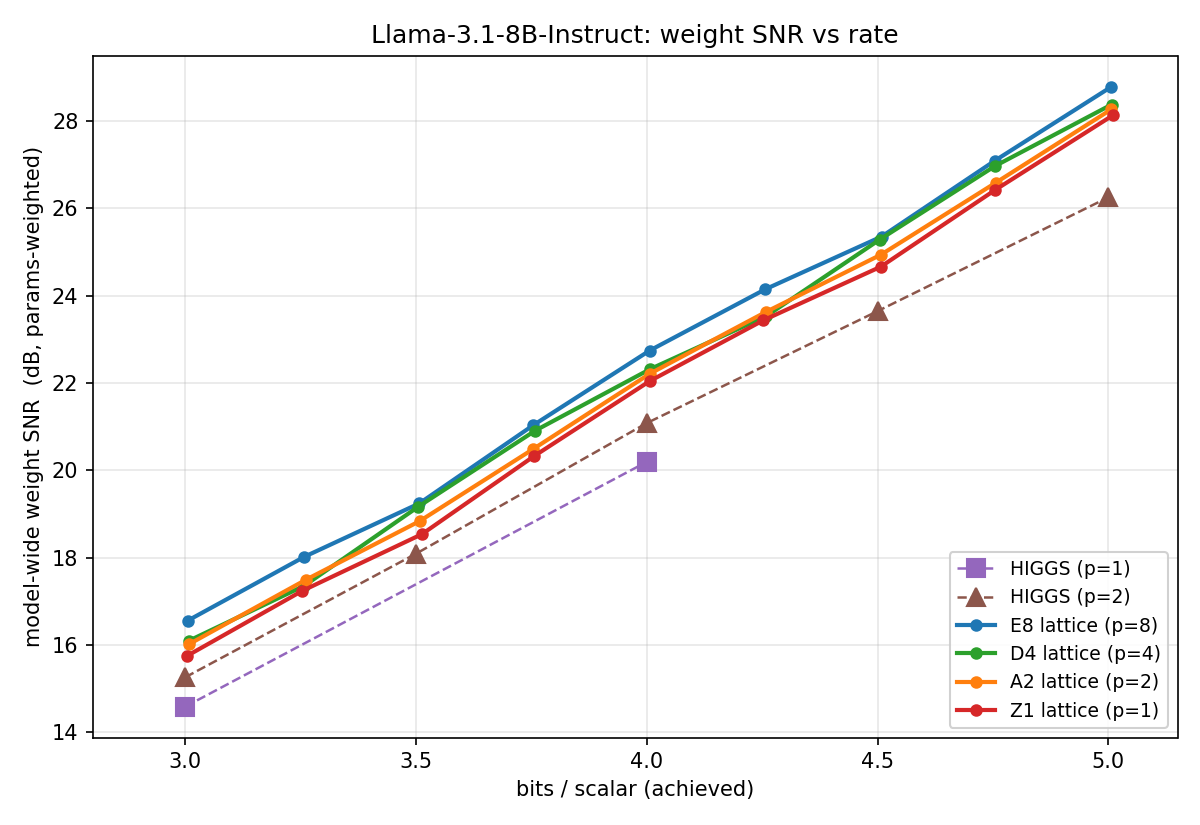}
  \caption{Weight SNR (dB) vs.\ bps. \method dominates by
  $1.3$--$2.5$~dB across the entire range.}
  \label{fig:weights-snr-bps}
\end{subfigure}\hfill
\begin{subfigure}{0.49\linewidth}
  \includegraphics[width=\linewidth]{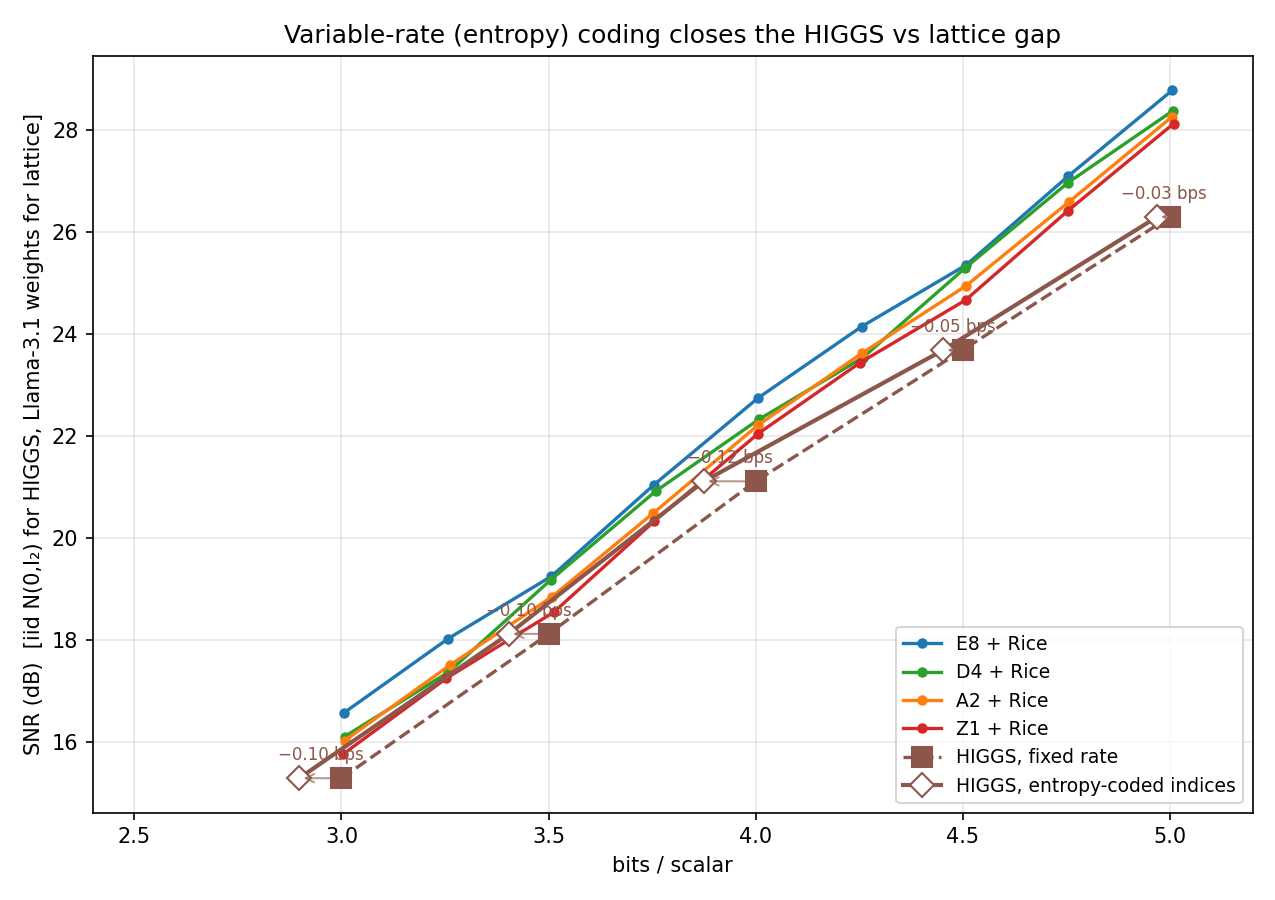}
  \caption{Entropy slack: fixed-rate budget minus the empirical
  index entropy of the HIGGS codebook. The lattice + Rice path
  closes this slack by entropy coding.}
  \label{fig:entropy-gain}
\end{subfigure}
\caption{The two components of the \method-vs-HIGGS gap.}
\end{figure}

\begin{table}[!ht]
\centering\small
\begin{tabular}{c c c c c c}
\toprule
bps & HIGGS-$p2$ SNR & $\Atwo$+Rice SNR & total $\Delta$ & entropy-coding piece & unbounded-codebook piece \\
\midrule
3.0 & 15.28 dB & 16.02 dB & 0.74 dB & 0.61 dB & \textbf{0.13 dB} \\
4.0 & 21.10 dB & 22.21 dB & 1.11 dB & 0.75 dB & \textbf{0.36 dB} \\
5.0 & 26.29 dB & 28.26 dB & 1.97 dB & 0.18 dB & \textbf{1.79 dB} \\
\bottomrule
\end{tabular}
\caption{Dimension-matched decomposition of the $\Atwo$-vs-HIGGS-$p2$
SNR gap (both 2-D). The ``entropy-coding piece'' is the rate slack
HIGGS would recover by re-encoding its existing index histogram with
a variable-length code, converted to dB at the local $6.02\,\dB$/bps
Gaussian R-D slope; it dominates at low rate. The
``unbounded-codebook piece'' is the residual that even an
entropy-coded HIGGS cannot recover, realizable only because the Rice
code lets the codebook be unbounded; it dominates at high
rate.}\label{tab:gap-decomp}
\end{table}

\paragraph{Higher-dimensional lattices.}
The $\Atwo$ comparison is deliberately conservative. HIGGS decodes by
table lookup, so its $2^{pb}$ codewords must fit in GPU shared
memory, which caps it at $p\in\{1,2\}$ in practice (a \bff $p{=}4$,
4~bps table already needs 512kiB). \method decodes
\emph{algebraically} with an $O(n)$ closest-point step and pays no
memory penalty for dimension, so it can use $\Dfour$ and $\Eone$.
Their lower Voronoi second moments, $G(\Eone)$ is only
$0.88\,\dB$ above the Shannon bound versus $\approx 1.3\,\dB$ for any
2-D grid (\Cref{tab:lattice-constants}), widen the SNR advantage to
$1.3$--$2.5\,\dB$ for $\Eone$ (\Cref{fig:weights-snr-bps}): pure
granular gain that is structurally unavailable to HIGGS. Concretely,
$\Eone$ drives weight-path PPL down to $8.744$ at 3~bps and $7.434$
at 4~bps (\bff baseline $7.161$), improving on the dimension-matched
$\Atwo$ ($9.273$ / $7.500$) by $0.53$ / $0.07$ PPL and on HIGGS-$p2$
($9.527$ / $7.618$) by $0.78$ / $0.18$ PPL, the margin largest in
the low-rate regime where granular gain dominates.



\subsubsection{KV cache: \method vs.\ TurboQuant / OCTOPUS}
\label{sec:cmp-octopus}

OCTOPUS~\citep{octopus2025} reports a KV-codec comparison (vs.\
TurboQuant and PolarQuant) on Qwen2.5-7B-Instruct-1M at context
$4096$ with symmetric $K{=}V$, measuring WikiText-2 and C4 PPL. We
reproduce that exact setting and match OCTOPUS's two bias-correction variants.
First, OCTOPUS notes a stability prerequisite: K-side protection on
the outer transformer blocks at each end. We confirm it independently:
with all K/V tiles quantized, \method (and any per-vector codec)
diverges on Qwen2.5-1M even at 4~bps ($\mathrm{PPL}\!\sim\!3500$)
despite $\sim\!22$~dB per-vector SNR. Keeping those two K tiles in
\bff{} (counted at $16$ bits in the rate) restores near-lossless
behavior. Second, OCTOPUS uses a $32$-token residual window (recent
K/V kept exact); we report \method both with and without it. We
implement the window per-query and charge its rate exactly:
at $W{=}32$, $T{=}4096$, it costs $\approx\!0.1$~bps and trims
$\mathrm{KV}\times$ only marginally. Because absolute WikiText-2
baselines differ between harnesses (ours $7.25$ vs.\ OCTOPUS's
$10.03$; C4 baselines agree to within $5\%$), the comparison is on
$\Delta\%$ relative to each method's own \bff{} baseline and on the
true compression $\mathrm{KV}\times = 16/\text{effective-bps}$.

\begin{table}[!ht]
\centering\small
\begin{tabular}{c l c c c c c}
\toprule
nom.\ bits & codec & corr. & res.\ win. & W2 $\Delta\%\downarrow$ & C4 $\Delta\%\downarrow$ & $\mathrm{KV}\times\uparrow$ \\
\midrule
\multirow{8}{*}{4\medskip}
  & TurboQuant-MSE & none & 32 & $+3.1$ & $+1.7$ & $2.2$ \\
  & TurboQuant-QJL & qjl  & 32 & $+8.0$ & $+7.9$ & $2.2$ \\
  & OCTOPUS        & none & 32 & $+2.7$ & $+1.5$ & $2.2$ \\
  & OCTOPUS-QJL    & qjl  & 32 & $+2.7$ & $+1.5$ & $2.0$ \\
  \rowcolor{blue!8} & \method        & none & -- & $+0.8$ & $+1.0$ & $3.7$ \\
  \rowcolor{blue!8} & \method        & qjl  & -- & $+1.4$ & $+1.0$ & $3.6$ \\
  \rowcolor{blue!8} & \method        & none & 32 & $\mathbf{+0.1}$ & $\mathbf{+0.2}$ & $3.6$ \\
  \rowcolor{blue!8} & \method        & qjl  & 32 & $+0.2$ & $+0.3$ & $3.5$ \\
\midrule
\multirow{8}{*}{3\medskip}
  & TurboQuant-MSE & none & 32 & $+8.6$  & $+8.3$  & $2.6$ \\
  & TurboQuant-QJL & qjl  & 32 & $+50.4$ & $+59.9$ & $2.5$ \\
  & OCTOPUS        & none & 32 & $+7.2$  & $+5.9$  & $2.5$ \\
  & OCTOPUS-QJL    & qjl  & 32 & $+7.2$  & $+6.1$  & $2.3$ \\
  \rowcolor{blue!8} & \method        & none & -- & $+5.5$  & $+5.7$  & $4.8$ \\
  \rowcolor{blue!8} & \method        & qjl  & -- & $+4.8$  & $+6.5$  & $4.6$ \\
  \rowcolor{blue!8} & \method        & none & 32 & $+1.8$  & $\mathbf{+1.4}$ & $4.6$ \\
  \rowcolor{blue!8} & \method        & qjl  & 32 & $\mathbf{+1.6}$ & $+1.5$ & $4.5$ \\
\midrule
\multirow{8}{*}{2\medskip}
  & TurboQuant-MSE & none & 32 & $+63.0$  & $+77.4$   & $3.0$ \\
  & TurboQuant-QJL & qjl  & 32 & $+772.0$ & $+1349.0$ & $3.0$ \\
  & OCTOPUS        & none & 32 & $+34.7$  & $+41.5$   & $2.9$ \\
  & OCTOPUS-QJL    & qjl  & 32 & $+34.7$  & $+41.4$   & $2.6$ \\
  \rowcolor{blue!8} & \method        & none & -- & $+42.0$  & $+54.3$   & $6.6$ \\
  \rowcolor{blue!8} & \method        & qjl  & -- & $+44.0$  & $+53.7$   & $6.4$ \\
  \rowcolor{blue!8} & \method        & none & 32 & $\mathbf{+7.4}$ & $\mathbf{+8.1}$ & $6.4$ \\
  \rowcolor{blue!8} & \method        & qjl  & 32 & $+14.7$  & $+15.2$   & $6.1$ \\
\midrule
 \rowcolor{blue!8} 1.7 & \method   & none & 32 & $+26.9$ & $+33.7$ & $7.1$ \\
\bottomrule
\end{tabular}
\caption{\method KV-only vs.\ OCTOPUS/TurboQuant on
Qwen2.5-7B-Instruct-1M (context $4096$, symmetric $K{=}V$). For each prior
codec we report both its no-bias-correction baseline (TurboQuant-MSE / OCTOPUS,
\texttt{none}) and its 1-bit-JL-residual variant (TurboQuant-QJL / OCTOPUS-QJL,
\texttt{qjl}), and \method both with and without the $32$-token residual window.
OCTOPUS/TurboQuant rows and their $\mathrm{KV}\times$ are from
\citep[Table~2]{octopus2025} and natively include the $32$-token window plus
K-side outer-block protection; \method rows are this work (WikiText-2 over $72$
windows, C4 \texttt{en}/validation over $39$ windows). $\mathrm{KV}\times$ is the
true compression (effective bits include the \bff-protected tiles and the
residual window). Bold marks the best $\Delta\%$ per bit-width
block.}\label{tab:octopus}
\end{table}

\Cref{tab:octopus} supports three conclusions.
\emph{(i) Matched bias-correction.} At matched scheme, \method's \texttt{qjl}
beats OCTOPUS-QJL ($+0.2\%$ vs.\ $+2.7\%$ at $4$ bits) and \texttt{none}
beats native OCTOPUS at every rate; the within-\method spread mirrors the
rotation-inversion at low bps.
\emph{(ii) Matched residual window.} With the same $32$-token window, \method
wins on both quality and compression at every operating point ($+7.4\%$ vs.\
OCTOPUS's $+34.7\%$ at $2$ bits, $\mathrm{KV}\times 6.4$ vs.\ $2.9$). The
window is decisive at $2$ bits: without it OCTOPUS leads on quality
($+34.7\%$ vs.\ $+42.0\%$), but it costs only $\approx0.1$ bps
(\Cref{tab:octopus-ablation}).
\emph{(iii) Compression.} The comparison is not compression-matched: OCTOPUS's
$2$-bit point uses $\approx5.5$ effective bits ($\mathrm{KV}\times2.9$) from
per-triplet norm overhead, vs.\ \method's $\approx2.5$ ($\mathrm{KV}\times6.4$).
On a compression-matched basis the advantage widens further, reaching
$\mathrm{KV}\times7.1$ at $1.7$~bps where OCTOPUS tops out near $3.0\times$.

\begin{table}[!ht]
\centering\small
\begin{tabular}{c c c c}
\toprule
bps & $\Delta\%$ (no window) & $\Delta\%$ (window $32$) & $\mathrm{KV}\times$ \\
\midrule
4   & $+0.8$   & $\mathbf{+0.1}$  & $3.7\to3.6$ \\
3   & $+5.5$   & $\mathbf{+1.8}$  & $4.8\to4.6$ \\
2   & $+42.0$  & $\mathbf{+7.4}$  & $6.6\to6.4$ \\
1.7 & $+325.1$ & $\mathbf{+26.9}$ & $7.5\to7.1$ \\
\bottomrule
\end{tabular}
\caption{Effect of the $32$-token residual window on \method
(\texttt{none}, WikiText-2 $\Delta\%$). The window adds $\approx0.1$
bps and so trims $\mathrm{KV}\times$ slightly, but the quality gain
grows sharply as the rate falls, exactly where recent-token
fidelity matters most.}\label{tab:octopus-ablation}
\end{table}

%% file: sections/ablation.tex
\section{Ablation study}\label{sec:ablation}

This section isolates the contribution of each \method component on
Llama-3.1-8B and identifies the parameters that materially move
quality. We organize by component; ablations we did not run are
folded into the future directions of \Cref{sec:conclusion}.

\subsection{Lattice choice}\label{sec:ablation-lattice}

We compare $\Z$, $\Atwo$, $\Dfour$, $\Eone$ on the weight path
(\Cref{fig:weights-bps}):

\begin{center}\small
\begin{tabular}{l c c c c}
\toprule
bps   & PPL($\Eone$) & PPL($\Dfour$) & PPL($\Atwo$) & PPL($\Z$) \\
\midrule
$3.0$ & $\mathbf{8.744}$ & $8.835$ & $8.973$ & $9.318$ \\
$4.0$ & $\mathbf{7.434}$ & $7.435$ & $7.466$ & $7.527$ \\
$5.0$ & $7.236$ & $\mathbf{7.227}$ & $7.249$ & $7.252$ \\
\bottomrule
\end{tabular}
\end{center}

Two practical conclusions:
\begin{itemize}[leftmargin=*,itemsep=2pt]
  \item Above 4.25~bps, all four lattices are essentially equivalent;
    use whichever has the simplest decoder ($\Z$ scalar is fine).
  \item Below 3.5~bps, $\Eone$'s additional granular gain becomes
    meaningful (up to $0.57$ PPL over $\Z$ at 3~bps): the only
    regime where the lattice choice has practical purchase.
\end{itemize}

\subsection{RHT tile size}\label{sec:ablation-hadamard}

The RHT tile is the block length over which we apply the RHT and
\texttt{amax}-scale before lattice quantization. We sweep it over
$\{128, 256, 512, 1024, 2048\}$ on Llama-3.1-8B with $\Eoneint$ at
3 and 4~bps on the weight, KV, and joint W$+$KV paths
(\Cref{fig:rht-tile}; WikiText-2 PPL over $141$ windows, \bff{}
baseline $7.161$). To separate a genuine tile effect from the
seed-to-seed noise of the random rotation, we repeat the sweep over
four independent RHT seeds and report the across-seed mean $\pm 1$
standard deviation. The SNR is calibrated once on iid-Gaussian data
and held fixed across tiles; the realized rate stays within $\pm
0.01$~bps of target, with larger tiles coding marginally fewer bits
($\approx\!0.005$~bps) by tightening the post-RHT Gaussian fit.

\begin{figure}[t]
\centering
\includegraphics[width=\linewidth]{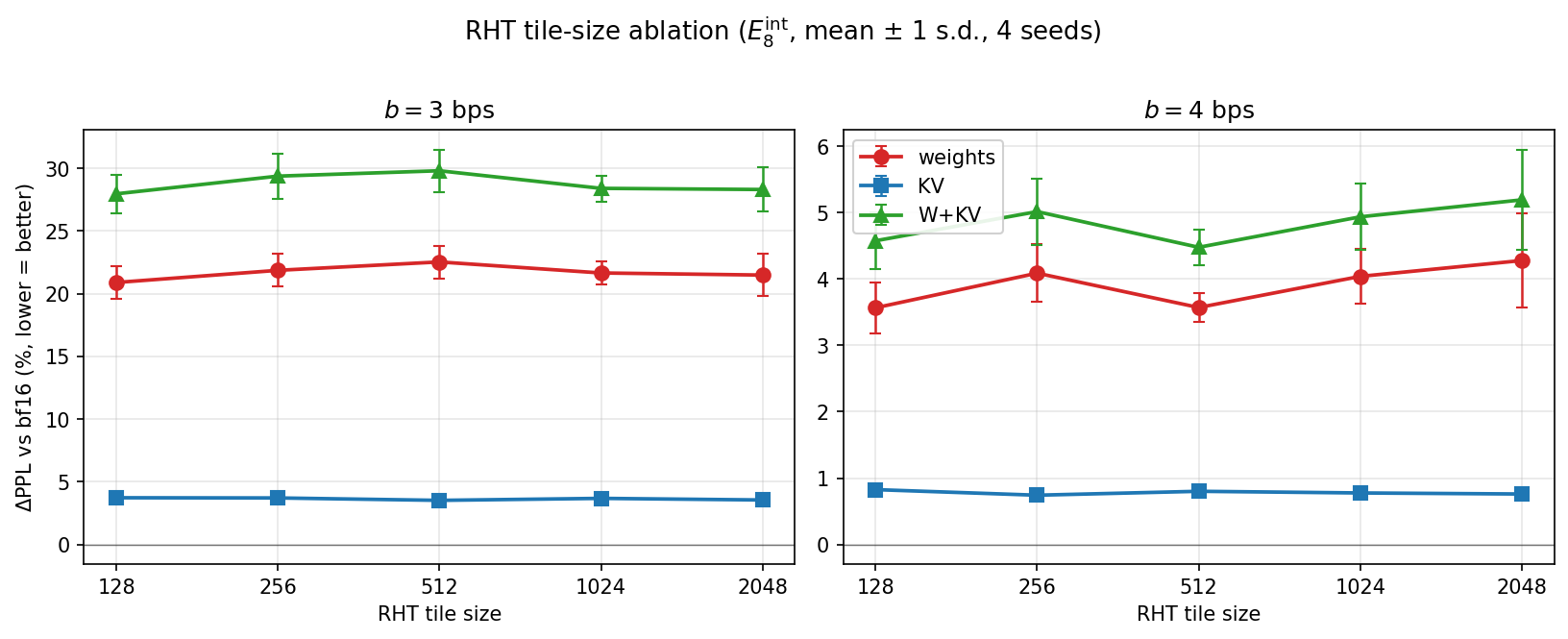}
\caption{PPL vs.\ RHT tile size on Llama-3.1-8B ($\Eoneint$, no MMA
cast, true RHT), as $\Delta$PPL relative to the \bff{} baseline for the
weight, KV, and joint W$+$KV paths at 3~bps (left) and 4~bps
(right). Markers are the mean over four RHT seeds; error bars are $\pm 1$
standard deviation. KV is nearly tile-invariant with negligible seed
variance, while for the weight and W$+$KV paths the per-tile differences
fall within the seed error bars at both rates, i.e.\ tile size is not
a quality lever. W$+$KV tracks the sum of the two independent
paths.}\label{fig:rht-tile}
\end{figure}

Two conclusions:
\begin{itemize}[leftmargin=*,itemsep=2pt]
  \item \textbf{KV is essentially tile-insensitive}, and this is the one
    rock-solid effect: the per-tile means span only $3.5$--$3.7\%$ at
    3~bps and $0.74$--$0.83\%$ at 4~bps, with tiny seed variance
    ($\le 0.24$ pp, mostly $<0.1$). KV vectors are low-dimensional and
    well-conditioned per head, so the RHT tile size barely
    matters once the tile covers a head or more. Per-head RHT
    (tile $=128$) is a fine, cheap default.
  \item \textbf{For weights and weights$+$KV, tile size is not a quality
    lever.} Across seeds the weight path stays at $\approx 21$--$22.5\%$
    at 3~bps and $\approx 3.6$--$4.3\%$ at 4~bps, with per-tile gaps
    ($\le 1.6$ pp) that are smaller than the $\pm 1$ s.d.\ seed noise
    ($\approx 1.3$ pp at 3~bps, $0.2$--$0.7$ pp at 4~bps); W$+$KV
    behaves the same way. Apparent ``best'' tiles from any single basis
    (e.g.\ a deterministic Hadamard, or one random seed) do not survive
    averaging over rotations, so the marginally better Gaussianization
    and lower rate of a longer transform do not translate into a
    reproducible PPL gain.
\end{itemize}

The W$+$KV damage tracks the sum of the independent paths (e.g.\ at
4~bps, tile $128$: weights $+3.57\%$ and KV $+0.83\%$ compose to
W$+$KV $+4.58\%$), confirming the two error sources are roughly additive
in PPL. Because the tile has no reproducible effect on quality, we set
it on kernel-efficiency grounds: the default tile of $128$ matches the
MMA $K$-dim and is therefore preferred for downstream kernel fusion at
no measurable accuracy cost.

\subsection{HIGGS-codebook efficiency analysis}\label{sec:ablation-higgs}

To corroborate the rate-gain decomposition in
\Cref{sec:results-weights}, we measured HIGGS's empirical codebook
index histograms on Llama-3.1-8B weights
(\Cref{fig:higgs-grid-usage}). The Lloyd grids exhibit clear
non-uniform usage: the most-used codeword is consistently $10$--$22\times$
more frequent than the least-used, and the empirical index entropy is
$0.6$--$5.9\%$ below the fixed-rate budget $\log_2 N$ (equivalently,
$94$--$99\%$ coding efficiency). This is
\emph{not} a bug in HIGGS; it is the expected behavior of any
finite codebook on a smooth distribution. It \emph{is} however a
recoverable bit-rate gap, which Rice coding closes.

\begin{figure}[t]
\centering
\begin{subfigure}{0.49\linewidth}
  \includegraphics[width=\linewidth]{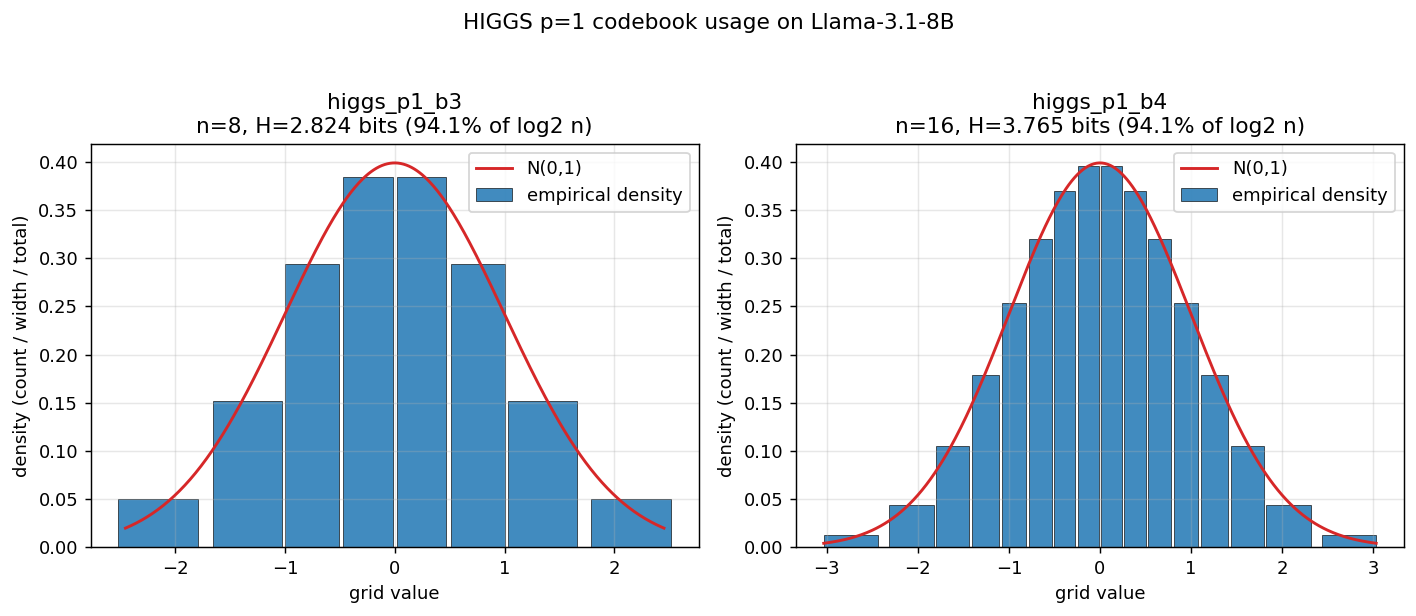}
  \caption{Index frequency for HIGGS $p{=}1$, 3~bps.}
\end{subfigure}\hfill
\begin{subfigure}{0.49\linewidth}
  \includegraphics[width=\linewidth]{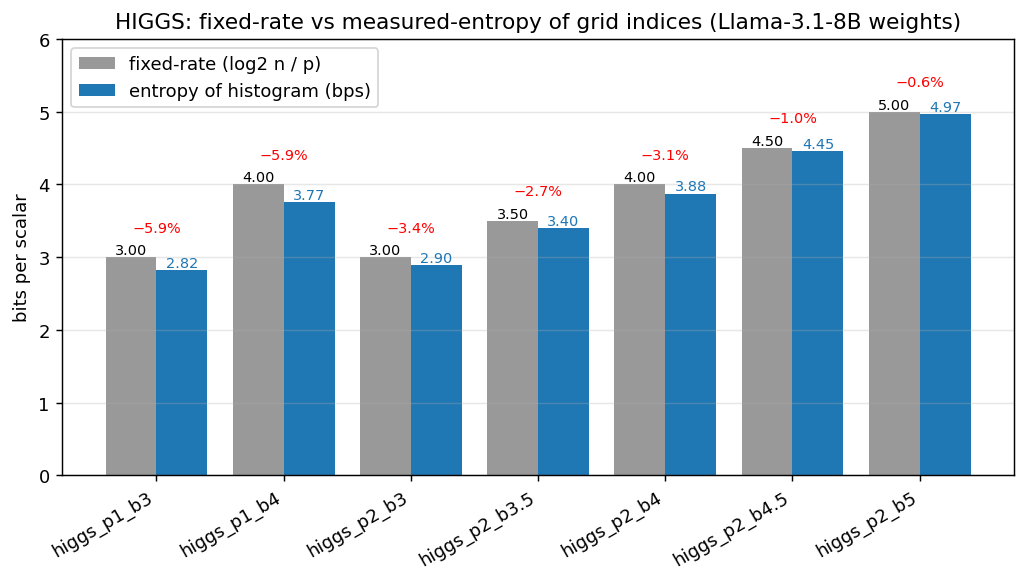}
  \caption{Empirical entropy efficiency $H/b$ for HIGGS codebooks
  at $b\in\{3,\ldots,5\}$, $p\in\{1,2\}$.}
\end{subfigure}
\caption{HIGGS codebook index distributions are non-uniform on
Llama weights, with $0.6$--$5.9\%$ entropy slack relative to the
fixed-rate budget $\log_2 N$ (red labels). \method recovers this
slack via Rice coding.}\label{fig:higgs-grid-usage}
\end{figure}

%% file: sections/conclusion.tex
\section{Conclusion and discussion}\label{sec:conclusion}

We presented \method, a data-free PTQ pipeline that unifies five
ingredients, per-tile RHT, optimal low-dimensional lattice
quantization, lossless bit-stripping, Rice entropy coding, and
Schuchman-Zamir-Feder subtractive dither, into a single recipe for
both the weights and the KV cache of modern transformers. The pipeline
plugs into Hopper's 8-bit and Blackwell's 4-bit MMA paths via a
per-tile \fpEight/\intEight cast, which is near-optimal once the RHT has
Gaussianized each tile's coordinates.

\paragraph{Summary of empirical findings.}
\begin{itemize}[leftmargin=*,itemsep=2pt]
  \item \emph{Weight quantization}: \method's $\Eone$+Rice path
    dominates HIGGS at every bps from 3 to 5. The gap
    decomposes into (i) a small index-entropy piece ($0.6$--$0.8$
    dB across the range) that any entropy-coded HIGGS could
    recover, and (ii) a larger structural ``unbounded-codebook''
    piece that even an entropy-coded HIGGS cannot match ($0.67$ dB
    at 3~bps, $0.91$ dB at 4~bps, $2.34$ dB at 5~bps),
    enabled by the variable-length coding that allows the lattice
    to be unbounded.
  \item \emph{KV-cache quantization}: a clean two-regime story
    emerges. Above 2.5~bps all bias-correction choices are equivalent;
    in the high-compression
    regime ($1.7$--$2.5$~bps) QJL rotation pulls ahead by up to
    $\sim 0.5$~PPL. Run head-to-head on OCTOPUS's own Qwen2.5-7B
    protocol, \method beats both TurboQuant and OCTOPUS at matched
    bias correction; with a matched $32$-token residual window it
    wins on \emph{both} quality and compression at every operating
    point ($+7.4\%$ vs.\ OCTOPUS's $+34.7\%$ perplexity at $2$ bits,
    at $\mathrm{KV}\times 6.4$ vs.\ $2.9$), and reaches
    $\mathrm{KV}\times 7.1$ at $1.7$ bps where OCTOPUS tops out near
    $3.0\times$.
  \item \emph{8-bit MMA}: \intEight consistently beats \fpEight on
    post-RHT lattice data by $\sim 0.1$ PPL (LLM) and
    $\sim 0.7$ dB PSNR (LTX-2 video), reversing the conventional
    wisdom that \fpEight is preferred for outlier-heavy distributions:
    post-RHT the distribution is no longer outlier-heavy.
  \item \emph{4-bit MMA}: \nvfp is viable; \mxfp's \Eeightmzero scale
    cannot accommodate KV-cache tails without dither rescue.
  \item \emph{Generalization}: the entire pipeline transfers
    cleanly from an 8B language model to a 19B video DiT.
\end{itemize}

\paragraph{When to use \method.} The defaults in \Cref{tab:params}
deliver $\Delta\ppl \le 0.3$ on Llama-3.1-8B at 4~bps for both
weights and KV with no fine-tuning, no calibration set, and a $\sim
30$-second post-training pass. For workloads where KV-cache memory
is the bottleneck (long-context decoding, batch inference, multi-tenancy)
we recommend 3~bps (KV), which delivers $\sim\!81\%$ KV memory
reduction for $+0.25$ PPL. For aggressive memory-constrained
deployments at 2~bps (KV), enable QJL rotation; below 1.7
bps the operating regime is too noisy for any data-free method we
know, and either calibration-based methods or fine-tuning is needed.

\paragraph{Limitations.}
\begin{enumerate}[leftmargin=*,itemsep=2pt]
  \item \emph{Memory win, not a speedup, on H100.} We measure
    $3.9\times$ weight compression ($2.8\times$ full-model resident)
    and $3.79\times$ KV-cache compression at near-lossless quality
    (\Cref{tab:throughput}), but the per-forward
    variable-length decode \emph{adds} latency rather than removing it,
    because \bff cuBLAS is already near roofline and a Rice stream cannot
    be fed into a tuned MMA mainloop. Turning the rate gain into a
    wall-clock speedup needs the kernel work below.
  \item \emph{Single-bps allocation.} \method currently uses
    uniform bps; the dynamic-programming bit allocator of HIGGS
    can be composed with our entropy code and should help in the
    very low-bit regime.
\end{enumerate}

\paragraph{Future directions.}
The most natural extensions, in order of expected impact:
\begin{enumerate}[leftmargin=*,itemsep=2pt]
  \item \emph{Kernel optimizations toward a throughput win.} Three
    directions build on the offset-indexed decoder
    (\Cref{sec:implementation}): (a)~\emph{intra-stream parallelism}
    via delta-coded sub-offsets, restoring occupancy at $\sim\!1\%$
    metadata overhead; (b)~\emph{warp-specialized fused decode$+$MMA},
    with producer warps decoding tiles into shared memory while
    consumer warps run \texttt{wgmma}, hiding decode under the matmul;
    (c)~\emph{rANS in place of Rice}~\citep{duda2013ans}, removing the
    serial unary scan for $N$-way SIMD decode.
  \item \emph{Close the FP-INT gap with a Gaussian-aware cast.}
    Post-RHT tiles are approximately iid Gaussian and light-tailed,
    making \intEight's uniform levels a better match than \fpEight's logarithmic
    spacing. An \fpEight cast designed for the known post-RHT density
    (e.g.\ companding or an analytic-tail saturation point) should
    close the $\sim\!0.1$~PPL/$\sim\!0.7$~dB gap without calibration,
    since the RHT fixes the marginal distribution data-free.
  \item \emph{Add a calibration pass such as LDLQ.} A one-shot
    LDLQ-style update~\citep{tseng2024quipsharp,savkin2025nestquant}
    adjusting each layer's unquantized weights to absorb prior
    quantization errors should close the residual gap to
    calibration-based methods with only a data-light pass over the
    model.
  \item \emph{Per-layer bit allocation.} Composing HIGGS's
    dynamic-programming allocator~\citep{malinovskii2025higgs} with
    our entropy code should concentrate gains in the very low-bit
    regime.
  \item \emph{Higher-dimensional lattices.} The Leech lattice
    $\Lambda_{24}$ offers a $\sim\!0.4$~dB granular-gain advantage
    over $\Eone$ at the cost of a more expensive decoder, to be
    weighed against its PPL benefit at 3, 3.5, and 4~bps.
\end{enumerate}

%% file: appendix/dither_proof.tex
\section{Proof of subtractive-dither unbiasedness}\label{app:dither_proof}

This appendix gives a self-contained proof that subtractive-dithered
lattice quantization satisfies the Schuchman conditions and is
therefore \emph{exactly} unbiased under inner products on every
realization: for any deterministic query $q$ and source $x$, the
dithered reconstruction $\hat{x}$ satisfies
$\E_U[\inner{q}{\hat x} \mid x] = \inner{q}{x}$, the
expectation being over the dither $U$ alone. We then show the
guarantee survives composition with the full \method KV pipeline,
and contrast it with the weaker approximate unbiasedness of
QJL without dither.

\subsection{Setup and notation}\label{app:dither-setup}

\begin{definition}[Lattice and Voronoi cell]
A lattice $\Latt \subset \R^n$ is a discrete subgroup of $(\R^n, +)$
of full rank $n$, i.e., $\Latt = \{\sum_{i=1}^n k_i b_i : k_i\in\Z\}$
for some $\R$-basis $b_1,\dots,b_n$ of $\R^n$. The Voronoi cell of
$\Latt$ at the origin is
\begin{equation}\label{eq:voronoi-def}
\Vor(\Latt) := \{x \in \R^n : \norm{x} \le \norm{x - \lambda}
\text{ for all } \lambda \in \Latt\}.
\end{equation}
\end{definition}

The Voronoi cell $\Vor(\Latt)$ is a closed convex polytope. It is
centrally symmetric, $\Vor(\Latt) = -\Vor(\Latt)$, because the
defining inequalities \eqref{eq:voronoi-def} are invariant under
$x \mapsto -x$ together with $\lambda \mapsto -\lambda$ (which is a
bijection of $\Latt$). The translates $\{\Vor(\Latt) + \lambda :
\lambda \in \Latt\}$ tile $\R^n$, overlapping only on the
measure-zero boundary $\partial\Vor(\Latt)$.

\begin{definition}[Fundamental domain]
A measurable set $D \subset \R^n$ is a fundamental domain for
$\Latt$ if $\R^n = \bigsqcup_{\lambda \in \Latt} (D + \lambda)$ up
to sets of Lebesgue measure zero.
\end{definition}

By the tiling property, $\Vor(\Latt)$ is itself a fundamental
domain. The covolume of $\Latt$ is $\vol(\Vor(\Latt)) =
|\det(b_1,\dots,b_n)|$.

\begin{definition}[Nearest-neighbor quantizer and mod-$\Latt$
projection]
Define
\[
\Qfn(y) := \arg\min_{\lambda \in \Latt} \norm{y-\lambda},
\qquad
\modproj(y) := y - \Qfn(y),
\]
with a fixed measurable tie-break rule on $\partial\Vor(\Latt)$.
The map $\modproj$ sends $y \in \R^n$ to the unique representative
of $y + \Latt$ in $\Vor(\Latt)$ (uniqueness up to the boundary).
\end{definition}

The map $\modproj$ has two properties we will use repeatedly:
\begin{enumerate}[label=(P\arabic*),leftmargin=*,nosep]
\item \emph{Range.} $\modproj(\R^n) \subseteq
\Vor(\Latt)$.\label{prop:range}
\item \emph{Lattice periodicity.} $\modproj(y+\lambda) =
\modproj(y)$ for all $\lambda \in \Latt$ and $y\in\R^n$, because
$\Qfn(y+\lambda) = \Qfn(y) + \lambda$. Hence $\modproj$ descends
to a well-defined map $\R^n/\Latt \to
\Vor(\Latt)$.\label{prop:periodic}
\end{enumerate}

\subsection{The mod-$\Latt$ pushforward is uniform}

\begin{lemma}\label{lem:pushforward}
Let $D \subset \R^n$ be a fundamental domain of $\Latt$ with finite
positive measure. If $Y \sim \Unif(D)$, then $\modproj(Y) \sim
\Unif(\Vor(\Latt))$.
\end{lemma}

\begin{proof}
Let $A\subseteq\Vor(\Latt)$ be measurable. We have
$\Pr[\modproj(Y) \in A] = \vol(\modproj^{-1}(A)\cap D)/\vol(D)$.
By \ref{prop:periodic},
$\modproj^{-1}(A) = \bigsqcup_{\lambda\in\Latt} (A+\lambda)$,
which tiles $A + \Latt$ exactly once when restricted to any
fundamental domain $D$. Hence $\vol(\modproj^{-1}(A)\cap D) =
\vol(A)$, so $\Pr[\modproj(Y) \in A] = \vol(A)/\vol(D) =
\vol(A)/\vol(\Vor(\Latt))$, which is the uniform-on-$\Vor(\Latt)$
probability.
\end{proof}

The lemma has the following ``shift-invariance'' consequence, which
is the engine of the proof.

\begin{corollary}[Crypto Lemma]\label{cor:crypto}
Let $U \sim \Unif(\Vor(\Latt))$. For every deterministic
$x \in \R^n$,
\[
\modproj(x + U) \sim \Unif(\Vor(\Latt)),
\quad\text{independent of $x$.}
\]
\end{corollary}

\begin{proof}
Since $\Vor(\Latt)$ is a fundamental domain, so is the translate
$x + \Vor(\Latt)$. The variable $Y := x + U$ has distribution
$\Unif(x+\Vor(\Latt))$, and \Cref{lem:pushforward} applies with
$D = x+\Vor(\Latt)$.
\end{proof}

\subsection{Subtractive dither produces an unbiased estimator}

\begin{definition}[Dithered reconstruction]\label{def:dithered}
Fix $x\in\R^n$ and let $U \sim \Unif(\Vor(\Latt))$ be independent of
any other randomness. The subtractive-dithered reconstruction of $x$
is
\begin{equation}\label{eq:xhat-def}
\hat{x}(x;U) := \Qfn(x+U) - U.
\end{equation}
The quantization error is $e(x;U) := \hat{x}(x;U) - x$.
\end{definition}

\begin{theorem}[Schuchman]\label{thm:schuchman}
With $U$ and $\hat{x}$ as in \Cref{def:dithered}, for every
$x\in\R^n$ the error $e(x;U)$ is distributed as $-U' \sim
\Unif(-\Vor(\Latt))$, independent of $x$. In particular, by central
symmetry, $e(x;U) \sim \Unif(\Vor(\Latt))$ as well, and
\begin{equation}\label{eq:error-mean-zero}
\E_U[e(x;U) \mid x] = 0.
\end{equation}
\end{theorem}

\begin{proof}
Expand the error:
\[
e(x;U) = \hat{x}(x;U) - x = \Qfn(x+U) - U - x =
-((x+U) - \Qfn(x+U)) = -\modproj(x+U).
\]
By the Crypto Lemma (\Cref{cor:crypto}),
$\modproj(x+U) \sim \Unif(\Vor(\Latt))$, so $e(x;U) =
-\modproj(x+U) \sim \Unif(-\Vor(\Latt))$ independent of $x$. The
mean is zero because $\Vor(\Latt)$ is centrally symmetric.
\end{proof}

\begin{corollary}[Inner-product unbiasedness]\label{cor:ip-unbiased}
For any deterministic $q\in\R^n$ and any source $x\in\R^n$,
\[
\E_U[\inner{q}{\hat{x}(x;U)} \mid x] = \inner{q}{x}.
\]
\end{corollary}

\begin{proof}
By linearity of expectation and \Cref{thm:schuchman},
$\E_U[\inner{q}{\hat{x}} \mid x] = \inner{q}{x} +
\inner{q}{\E_U[e\mid x]} = \inner{q}{x}+ \inner{q}{0} =
\inner{q}{x}$.
\end{proof}

\begin{corollary}[Variance bound]\label{cor:variance}
With the same notation, and writing $\sigma^2_\Vor :=
\frac{1}{n}\E_{U\sim\Unif(\Vor)}[\norm{U}^2]$ for the
per-coordinate second moment of the Voronoi cell,
\[
\Var_U(\inner{q}{\hat{x}} \mid x) = q^\top \Cov_U(U) q \le
\lambda_{\max}(\Cov_U(U)) \norm{q}^2.
\]
If $\Vor(\Latt)$ is isotropic, i.e., $\Cov(U) = \sigma^2_\Vor I_n$,
this becomes
$\Var_U(\inner{q}{\hat{x}} \mid x) = \sigma^2_\Vor \norm{q}^2$.
\end{corollary}

\begin{remark}\label{rem:tightness}
For the lattices used in \method ($\Z$, $\Atwo$, $\Dfour$,
$\Eone$), the Voronoi cell is isotropic — the lattice's
symmetry group acts irreducibly on $\R^n$, and by Schur's lemma
any invariant rank-2 tensor is a scalar multiple of $I_n$. Hence
$\Cov(U) = \sigma^2_\Vor I_n$ exactly. Numerically,
$\sigma^2_\Vor(\Z) = 1/12 = 0.0833$ and
$\sigma^2_\Vor(\Eone) \approx 0.287$ in our scaling.
\end{remark}

\subsection{Composition with the \method KV pipeline}

\begin{proposition}\label{prop:composition}
Let $R$ be any orthogonal matrix ($R^\top R = I$), let $\alpha > 0$
be the lattice's calibration scale, and let $\hat{x}$ be the
\method KV reconstruction of $x$:
\begin{align*}
s(x) &:= \frac{\alpha\sqrt{n}}{\norm{x}} \cdot Rx, \\
\widehat{s} &:= \Qfn(s(x) + U) - U, \\
\hat{x} &:= \frac{\norm{x}}{\alpha\sqrt{n}} R^\top \widehat{s}.
\end{align*}
Then $\E_U[\hat{x} \mid x] = x$, and in particular
$\E_U[\inner{q}{\hat{x}}\mid x] = \inner{q}{x}$ for every
deterministic $q\in\R^n$.
\end{proposition}

\begin{proof}
Apply \Cref{thm:schuchman} in the lattice's coordinate system to
$s := s(x)$: $\E_U[\widehat{s}\mid s] = s$. The post-quantization
map $\widehat{s} \mapsto (\norm{x}/(\alpha\sqrt{n}))R^\top \widehat{s}$
is linear and depends only on $x$ (not on $U$), so by linearity of
conditional expectation, $\E_U[\hat{x}\mid x] =
(\norm{x}/(\alpha\sqrt{n})) R^\top s(x) = R^\top R x = x$.
Inner-product unbiasedness follows.
\end{proof}

\Cref{prop:composition} holds for any orthogonal $R$:
deterministic identity, deterministic permutation, random Haar,
random sign diagonal; including when $R$ is itself random but
\emph{independent} of $U$, because the proof conditions on $R$ and
then averages.

\subsection{Contrast: QJL alone is biased per-vector}

The QJL-without-dither variant uses random rotation but no dither,
so its reconstruction is
\begin{equation}\label{eq:qjl-def}
\hat{x}_{\mathrm{QJL}}(x;S) := S^\top \Qfn(Sx),
\qquad S \sim \Unif(\mathrm{O}(n)),
\end{equation}
with error $e_{\mathrm{QJL}}(x;S) = -S^\top \modproj(Sx)$.

\begin{proposition}\label{prop:qjl}
For every fixed realization $S = S_0 \in \mathrm{O}(n)$ and every
fixed source $x \in \R^n$, the QJL error is a deterministic vector
$-S_0^\top \modproj(S_0 x)$, generally non-zero, so
$\hat{x}_{\mathrm{QJL}}(x;S_0)$ is a biased estimator of $x$. There
is no per-vector analog of \Cref{cor:ip-unbiased} for QJL alone.
\end{proposition}

\begin{proof}
Self-evident from \eqref{eq:qjl-def}: with $S_0$ fixed, neither
side of the equation depends on any further randomness.
\end{proof}

\paragraph{Why the empirical QJL bias appears small.} A typical
sweep test
$\frac{1}{N}\sum_{i=1}^N \inner{q^{(i)}}{e_{\mathrm{QJL}}(x^{(i)};S_0)}$
averages over many iid $(q^{(i)}, x^{(i)})$. Because $q$ is independent of
everything else and zero-mean, this average converges
to $\inner{\E q}{\cdot} = 0$. That follows from $\E q = 0$,
\emph{not} from any QJL property: a non-rotated lattice quantizer
passes the same test. What QJL \emph{does} provide is
approximately isotropic error covariance, $S_0^\top \Cov_x(\modproj(S_0
x))S_0$ close to a scalar multiple of $I_n$ for a generic Haar $S_0$.
This is the genuine benefit of the random rotation, but it is not
unbiasedness.

\subsection{Practical sampler}\label{app:dither-sampler}

The proof requires $U\sim\Unif(\Vor(\Latt))$. The implementation
uses the ``mod-$\Latt$ trick'':
\[
U_{\mathrm{cube}} \sim \Unif([-2,2)^n),
\quad
U := \modproj(U_{\mathrm{cube}}).
\]
This is exact iff $[-2,2)^n$ is itself a fundamental domain of
$\Latt$ (\Cref{lem:pushforward}). For $\Eone$, the cube
$[-2,2)^8$ is not a union of $\Latt$-translates of
$\Vor(\Latt)$, so the projection is only \emph{approximately}
uniform; we validate the sampler by checking that
$\frac{1}{n}\E\norm{U}^2$ matches the
analytical $\sigma^2_\Vor$ to within $5\%$ on all four lattices.
An exact alternative is the fundamental-parallelepiped sampler:
draw $T_i\stackrel{\text{iid}}{\sim}\Unif[0,1)$ for $i=1,\dots,n$,
form $Y = \sum T_i b_i$ over the lattice basis $\{b_i\}$, and
project $U := \modproj(Y)$. This is exactly uniform on $\Vor(\Latt)$
by \Cref{lem:pushforward}, since the parallelepiped is by
construction a fundamental domain.

%% file: appendix/lattice_constants.tex
\section{Calibration: setting the operating point}
\label{app:calibration}

\method exposes one user-facing knob, the target rate $b$ in
bits/scalar, and turns it into a concrete quantizer in two steps:
choose the quantization SNR that yields $b$, then set the per-vector
scale $\alpha$ that realizes that SNR. The first step needs an
empirical rate curve $b(\snr)$, since the Rice-coded rate has no
closed form; the second is closed-form in the lattice's
Voronoi second moment. Both are built once on synthetic
iid-Gaussian data and, crucially, apply unchanged to every
weight and KV tensor in any model (\Cref{app:calibration-transfer}).
This is what lets \method hit an \emph{arbitrary} fractional rate,
which fixed-rate codebooks cannot.

\subsection{From a target rate to an SNR (empirical)}
\label{app:bps-calibration}

The rate of the stripped, Rice-coded stream combines
the lattice's granular gain, the bit-stripping transform, an
\emph{integer}-parameter Rice coder, and the 8-bit clip, none of
which has a clean closed form at the operating points of
interest. We therefore measure it: for each lattice we draw
$N \sim 10^{5}$ iid-Gaussian tiles, quantize at a grid of SNRs
(each set by the closed-form $\alpha$ of \Cref{app:alpha-calibration}),
and record the \emph{realized} Rice rate (\Cref{fig:rate-vs-snr}).

\begin{figure}[h]
\centering
\includegraphics[width=0.72\linewidth]{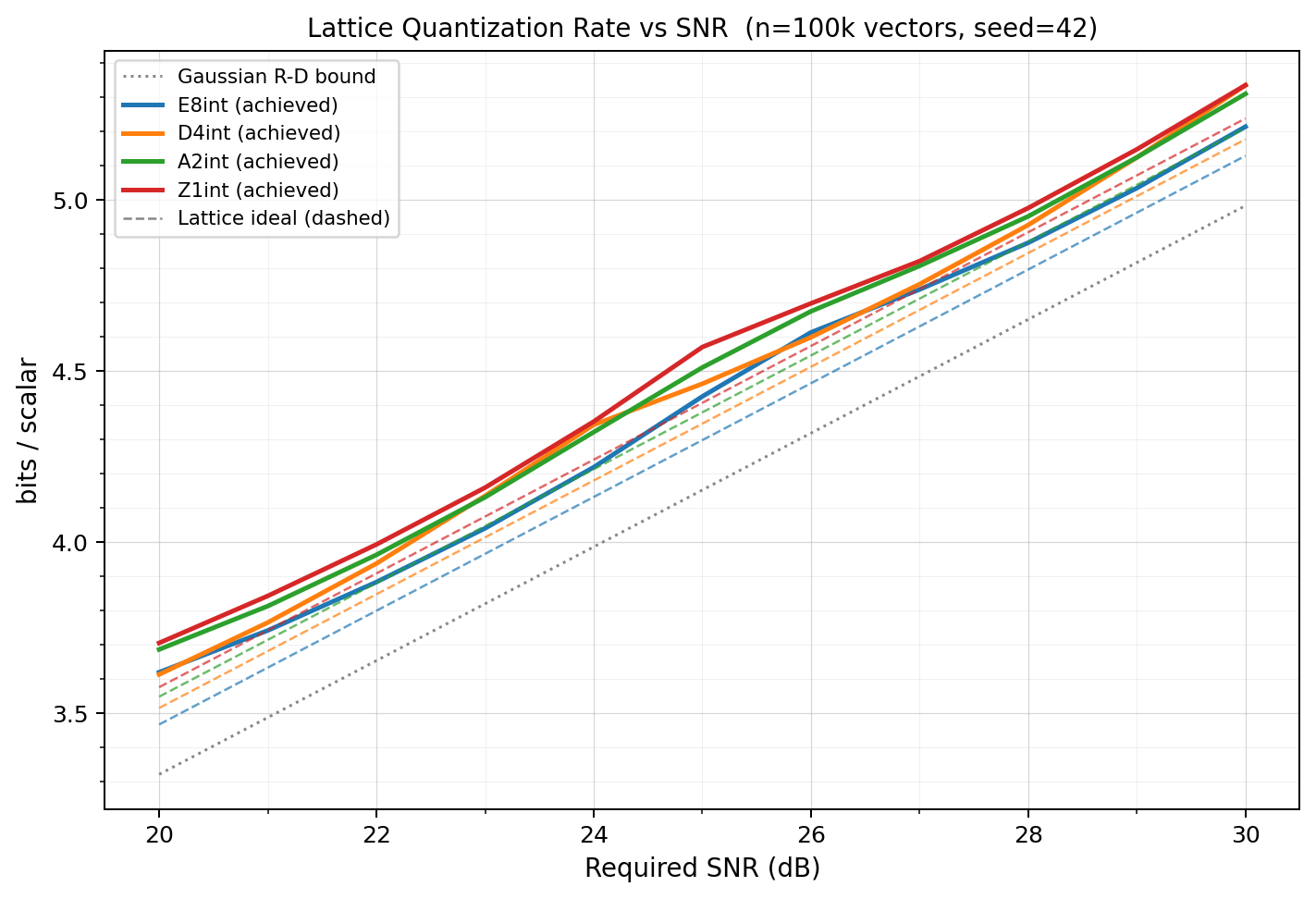}
\caption{Empirical Rice rate versus target SNR on iid-Gaussian
tiles ($N{=}10^{5}$, seed~42), for the four lattices. Solid:
achieved bits/scalar; dashed: the lattice ideal
$R_D + \tfrac12\log_2(2\pi e\,G(\Latt))$; dotted: the Gaussian
rate-distortion bound $R_D = \tfrac12\log_2 \mathrm{SNR_{lin}}$.
Each curve is smooth and monotone, so inverting it sends any target
rate to a unique SNR; the achieved rate stays $\approx 0.1$~bps
above the lattice ideal throughout, ordering
$\Eone < \Dfour < \Atwo < \Z$.}\label{fig:rate-vs-snr}
\end{figure}

Two properties make this usable. \emph{(i)~Invertibility.} $b(\snr)$
is monotone increasing, so the implementation interpolates the
tabulated curve to recover the unique SNR achieving any requested
$b$; because the table stores the realized Rice rate (not an entropy
estimate), selecting against it hits any target to within
$\sim\!0.01$~bps (\Cref{sec:results-weights} confirms this on-model).
\emph{(ii)~Tightness.} The realized rate sits
$\approx\!0.1$~bps above the lattice ideal. This gap is almost
entirely the redundancy of the stateless, power-of-two Rice code
over the symbols' \emph{marginal} entropy; that marginal entropy
already meets the lattice ideal (\Cref{app:strip-optimality}), so
the $0.1$~bps reflects the coder's simplicity, not residual
lattice or inter-symbol inefficiency.

The table is the rate \emph{of} the per-lattice Rice structure:
$\Eone$ uses a single parameter $k_s$ (the coset bit
folded into the combined symbol); $\Dfour$ and $\Atwo$ use two
($k_s$ with $k_t = k_s - 1$ for $\Dfour$, and $k_{t_y}, k_{n_x}$ for
$\Atwo$); $\Z$ uses one. The halving identity $k_t = k_s - 1$ is
checked at runtime (\Cref{sec:method-rice}).

\subsection{From an SNR to the scale (closed form)}
\label{app:alpha-calibration}

Given the SNR, the scale is analytic. For $x \sim \mathcal{N}(0,
I_N)$ the high-rate quantization error has mean square equal to the
lattice's Voronoi second moment, $\mathrm{MSE}_{\mathrm{vor}} =
G(\Latt)\,N\,V_\Latt^{2/N}$, while the signal power is
$\E\norm{\alpha x}^2 = \alpha^2 N$. Setting their ratio to the
target $\mathrm{SNR_{lin}} = 10^{\snr/10}$ yields
\begin{equation}\label{eq:alpha-calib}
  \alpha(\snr, \Latt)
  \;=\; \sqrt{\frac{\mathrm{SNR_{lin}}\,\mathrm{MSE}_{\mathrm{vor}}}{N}}
  \;=\; \sqrt{\mathrm{SNR_{lin}}\;G(\Latt)\,V_\Latt^{2/N}},
\end{equation}
with $V_\Latt$ the covolume of the integer realization. No
calibration data are needed; the only assumption is the
high-rate approximation, which we verify: the empirical Voronoi
second moment is within $5\%$ of $\mathrm{MSE}_{\mathrm{vor}}$ for
$\Eone, \Dfour, \Atwo$ (within $20\%$ for $\Z$, whose tiny cell is
sensitive to \bff rounding), and the realized SNR lands within
$\approx\!0.1$~dB of the target across $20$--$30$~dB.

\paragraph{The code fits a signed byte.} \Cref{eq:alpha-calib} also
fixes the code magnitudes, hence whether the stored integer
coordinates fit \intEight. After the per-tile RHT and $\ell_2$
normalization each input scalar has unit variance
(\Cref{app:calibration-transfer}); the quantizer rounds $\alpha u$ to the
lattice, so each stored coordinate tracks $\alpha u_i$, giving
\begin{equation}\label{eq:int8-gaussian}
  Y_i \;\approx\; \mathcal{N}(0, \alpha^2),
\end{equation}
with the granular noise adding only $O(G(\Latt))$ to the variance, negligible
against $\alpha^2$. An overflow ($|Y_i| > 127$) is therefore a
$127/\alpha$-sigma tail event, with per-coordinate probability
$\mathrm{erfc}\!\big(127/(\alpha\sqrt 2)\big)$.

Bit-stripping (\Cref{sec:method-lattices,app:strip-optimality}) only
\emph{shrinks} magnitudes: for $\Eone$/$\Dfour$ the stored $s$-coordinates
have $Y\!\gg\!1$ and the parity coordinate is halved again, so the only
un-shrunk coordinates are $\Z$'s scalars and $\Atwo$'s $n_x$. The raw
lattice integer is thus the binding case, and bounding it bounds the
entropy-coded symbols a fortiori. The binding \emph{rate} is the top of
the sweep, $5$~bps, where $\alpha$ (and with it the code magnitude) is
largest; \Cref{tab:int8-budget} evaluates \eqref{eq:int8-gaussian} there.

\begin{table}[h]
\centering\small
\begin{tabular}{l c c c c}
\toprule
Lattice & $\alpha$ & $127/\alpha$ & $P(|Y_i| > 127)$ & exp.\ overflows / $7{\times}10^{9}$ \\
\midrule
$\Eone$         & $14.7$ & $8.6\sigma$ & $6\times10^{-18}$ & $4\times10^{-8}$ \\
$\Dfour$        & $17.3$ & $7.3\sigma$ & $2\times10^{-13}$ & $1\times10^{-3}$ \\
$\Atwo$ ($n_x$) & $13.6$ & $9.3\sigma$ & $1\times10^{-20}$ & $7\times10^{-11}$ \\
$\Z$            & $7.4$  & $17\sigma$  & ${<}10^{-60}$     & $\approx 0$ \\
\bottomrule
\end{tabular}
\caption{Worst-case \intEight overflow at the top of the sweep ($5$~bps),
the rate at which $\alpha$, and hence the code magnitude, is largest.
With $Y_i \approx \mathcal{N}(0,\alpha^2)$ an overflow is a
$127/\alpha$-sigma event; the last column multiplies the per-coordinate
tail by the $\sim\!7\times10^{9}$ quantized weights of Llama-3.1-8B. The
binding lattice is $\Dfour$ (largest $\alpha$), still $7.3\sigma$ from the
byte boundary.}\label{tab:int8-budget}
\end{table}

Even the worst lattice, $\Dfour$, expects only $\sim\!10^{-3}$ saturations
across the whole model; the default $\Eone$ sits $8.6\sigma$ out, at
$\sim\!4\times10^{-8}$. And a saturation, when it occurs, is a clamp to
$\pm127$ that perturbs one coordinate by less than the granular step, not
a corruption. The analytic tail is moreover conservative: empirically
(\Cref{fig:e8-budget}) the largest coordinate over the full
$20$--$30$~dB sweep is $\approx\!3.4\sigma$ ($|Y|\!\approx\!50$ for
$\Eone$), because the per-tile $\ell_2$ normalization caps the maximum
near $\sqrt{2\ln 1024}\approx 3.7\sigma$, tighter than a free Gaussian.
This justifies storing the raw code in one signed byte per scalar: the
\intEight Tensor-Core format and the entropy-free fallback.

\begin{figure}[h]
\centering
\includegraphics[width=0.72\linewidth]{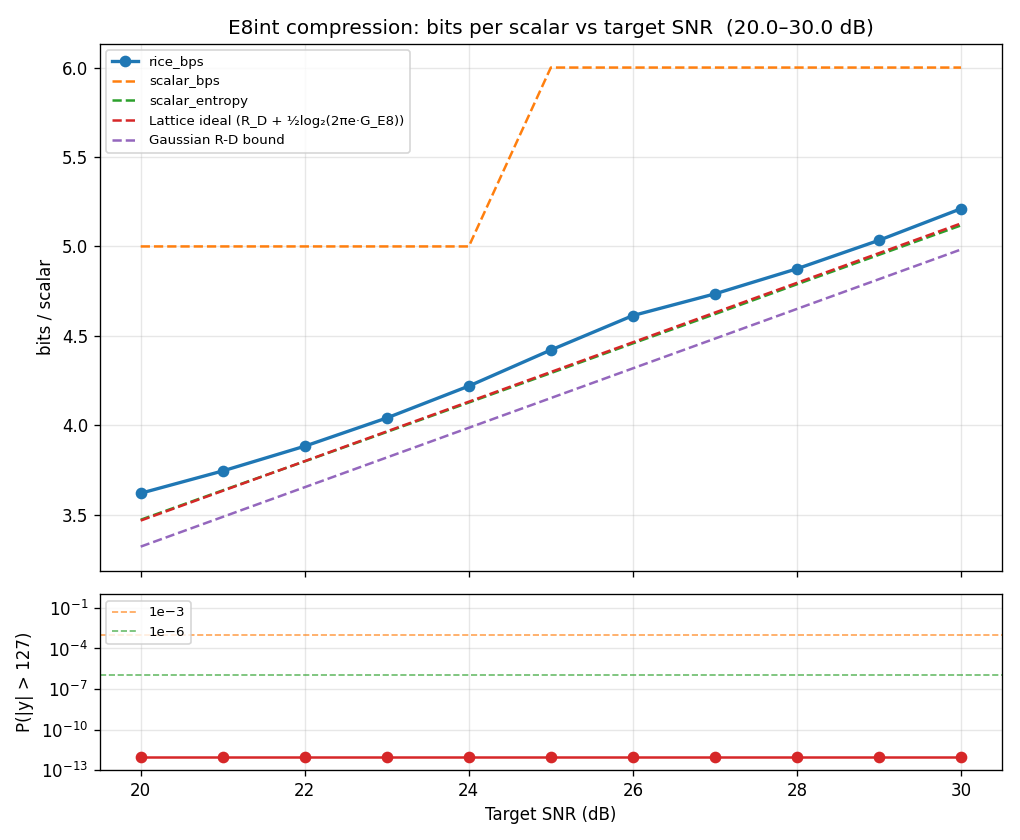}
\caption{$\Eone$ calibration detail. \emph{Top:} achieved Rice rate
tracks the lattice ideal within $\approx 0.1$~bps over $20$--$30$~dB
(the dashed \texttt{scalar\_bps} is the byte-aligned \intEight fallback,
which steps with the integer byte budget; the \texttt{scalar\_entropy}
curve, the marginal entropy of the stripped indices, lies
essentially on the lattice ideal, cf.\ \Cref{app:strip-optimality}).
\emph{Bottom:} sampled $\intEight$ overflow rate $P(|c_i| > 127)$, which
is $0$ at every operating point (markers floored at $10^{-12}$ to render
on the log axis); the largest coordinate seen is $\approx\!3.4\sigma$,
well inside the $\ge\!7.3\sigma$ analytic margin of
\Cref{tab:int8-budget}.}\label{fig:e8-budget}
\end{figure}

\subsection{One calibration for all data}
\label{app:calibration-transfer}

Both pieces are built on iid-Gaussian tiles, yet \method
applies them to real weights and KV activations with no per-tensor
recalibration. The justification is the per-tile RHT followed by
$\ell_2$ normalization (\Cref{sec:method-rht,sec:method-normalize}):
the RHT spreads each tile's energy across its
coordinates and normalization fixes its radius, so every tile is
approximately an isotropic Gaussian on the sphere of radius
$\alpha\sqrt{n}$, precisely the distribution the calibration was
measured on. A single table and the closed form of
\Cref{eq:alpha-calib} therefore serve all tensors, layers, and
models, which is what makes \method data-free.

\subsection{Stripping is rate-optimal: marginal entropy meets the lattice ideal}
\label{app:strip-optimality}

A striking feature of the sweep is that the per-scalar
\emph{marginal} entropy of the stripped symbols
($\frac1N\sum_i H(s_i)$, reported as \texttt{scalar\_entropy})
coincides with the lattice ideal
$R_{\mathrm{ideal}} = R_D + \tfrac12\log_2\!\bigl(2\pi e\,G(\Latt)\bigr)$
to within $\sim\!10^{-2}$~bps at every operating point and lattice
(\Cref{fig:e8-budget}, top, where the two curves overlie).
This is no coincidence: it follows from composing two classical
high-resolution results with the design of the strip.

\paragraph{(i) The index entropy equals the lattice ideal.} For a
source $X$ with finite differential entropy quantized by a lattice
$\Latt$ at fine resolution, the index entropy obeys the
high-resolution law
\begin{equation}\label{eq:gishpierce}
  H\bigl(\Qfn(X)\bigr) \;=\; h(X) \;-\; \log_2 \operatorname{vol}(\Vor(\Latt)) \;+\; o(1),
\end{equation}
the $o(1)$ vanishing as the cell shrinks
\citep{gish1968asymptotically,lookabaugh1989high}. With per-dimension
distortion $D = G(\Latt)\,\operatorname{vol}(\Vor(\Latt))^{2/N}$
\citep[Ch.~3]{conway1999sphere} and a white Gaussian source
$X \sim \mathcal N(0,\sigma^2 I_N)$, for which
$h(X)/N = \tfrac12\log_2(2\pi e\,\sigma^2)$, dividing
\eqref{eq:gishpierce} by $N$ gives
\begin{equation}\label{eq:idealrate}
  \frac1N H\bigl(\Qfn(X)\bigr)
  \;\longrightarrow\;
  \underbrace{\tfrac12\log_2\frac{\sigma^2}{D}}_{R_D}
  \;+\; \tfrac12\log_2\!\bigl(2\pi e\,G(\Latt)\bigr)
  \;=\; R_{\mathrm{ideal}}.
\end{equation}
The excess $\tfrac12\log_2(2\pi e\,G(\Latt))$ over the Gaussian
rate-distortion bound $R_D$ is the lattice's space-filling loss,
exactly the redundancy of an entropy-coded dithered lattice
quantizer above $R(D)$
\citep{zamir1992universal,erez2005closeness,zamir2014lattice}.

\paragraph{(ii) Stripping reduces the marginal sum to the joint entropy.}
The quantity \texttt{scalar\_entropy} is not the joint entropy
\eqref{eq:idealrate} but the per-scalar \emph{sum of marginals},
$\frac1N\sum_i H(s_i)$. The strip $\mathrm{Strip}_\Latt$ is a
lossless bijection $\Latt \leftrightarrow \mathbb Z^{N}$
(\Cref{sec:method-lattices}), hence preserves the joint entropy,
$H(s_1,\dots,s_N) = H(\Qfn(X))$, and by subadditivity
\begin{equation}\label{eq:totalcorr}
  \frac1N\sum_i H(s_i)
  \;=\; \frac1N H\bigl(\Qfn(X)\bigr) \;+\; \frac{C}{N},
  \qquad
  C \;=\; \sum_i H(s_i) - H(s_1,\dots,s_N) \;\ge\; 0,
\end{equation}
with $C$ the total correlation (multi-information) of the symbols
\citep[Ch.~2]{cover2006elements}. Two design choices send
$C \to 0$. First, the strip is built to annihilate \emph{exactly}
the lattice's deterministic dependencies, the parity and coset
constraints of \Cref{sec:method-lattices}, the only exact couplings
among the integer coordinates. Second, the residual
\emph{statistical} dependence vanishes in the high-rate limit: as
the cell shrinks $\Qfn(X) \to X$, so each stripped symbol converges
to a scaled copy of an i.i.d.\ Gaussian source coordinate
($s_i \to \alpha X_i / 2$ for $\Eoneint$, and likewise for the
others) and the symbols become mutually independent. The subtractive
dither of \Cref{cor:crypto} makes this precise: it renders the
quantization error independent of $X$
\citep{zamir1992universal,schuchman1964dither}, removing the
input-dependent part of the residual correlation at any rate. Hence
$C/N \to 0$ and, combining \eqref{eq:totalcorr} with
\eqref{eq:idealrate}, $\frac1N\sum_i H(s_i) \to R_{\mathrm{ideal}}$.

\paragraph{Residual gap and consequence.} The measured discrepancy is
$\lesssim 0.02$~bps and changes sign: $C/N \ge 0$ pushes
\texttt{scalar\_entropy} above $R_{\mathrm{ideal}}$, while
finite-rate corrections to \eqref{eq:gishpierce} and the
$\approx\!0.06$~dB \bff SNR deficit push it below; both effects are
$O(10^{-2})$. Because the stripped symbols are statistically
near-independent and their marginal entropy sits at the lattice ideal,
a memoryless coder is near rate-optimal: no context or joint coder can
recover more than the vanishing $C/N$. Stripping is thus rate-optimal
by construction, not a heuristic, letting the Rice coder
(\Cref{sec:method-rice}) concede only $\sim\!0.1$~bps to the ideal.